\documentclass[twocolumn]{autart}
\usepackage[utf8]{inputenc}

\usepackage{fancyhdr}

\usepackage{amsmath,enumerate,color,multirow,hhline}
\usepackage{amsfonts}

\usepackage{xcolor}

\usepackage{algorithm}
\usepackage{algpseudocode}

\DeclareMathOperator{\argmin}{arg\,min}

\counterwithin*{section}{part}

\usepackage[final]{hyperref}

\hypersetup{breaklinks=true,colorlinks=true,linkcolor=blue,citecolor=blue}

\usepackage{graphicx}
\usepackage{graphics}
\usepackage[noabbrev,capitalise,nameinlink]{cleveref}
\usepackage{subcaption}

\usepackage[sort]{cite}
\usepackage{caption}

\makeatletter
\DeclareRobustCommand{\qed}{%
  \ifmmode % if math mode, assume display: omit penalty etc.
  \else \leavevmode\unskip\penalty9999 \hbox{}\nobreak\hfill
  \fi
  \quad\hbox{\qedsymbol}}
\newcommand{\openbox}{\leavevmode
  \hbox to.77778em{%
  \hfil\vrule
  \vbox to.675em{\hrule width.6em\vfil\hrule}%
  \vrule\hfil}}
\newcommand{\qedsymbol}{\openbox}
\newenvironment{proof}[1][\proofname]{\par
  \normalfont
  \topsep6\p@\@plus6\p@ \trivlist
  \item[\hskip\labelsep\itshape
    #1.]\ignorespaces
}{%
  \qed\endtrivlist
}
\newcommand{\proofname}{Proof}
\makeatother
\edef\endfrontmatter{%
  \unexpanded\expandafter{\endfrontmatter}% the current code
  \noexpand\endNoHyper % add \endNoHyper at the end to match \NoHyper
}
\newtheorem{theorem}{Theorem}
\newtheorem{example}{Example}
\newtheorem{assumption}{Assumption}
\newtheorem{remark}{Remark}

\newtheorem{lemma}{Lemma}
\newtheorem{definition}{Definition}

\usepackage{xcolor}
\usepackage{blindtext}%

\definecolor{ngreen}{RGB}{0,160,0}

\begin{document}
\begin{frontmatter}
\title{Analysis of Off-Policy $n$-Step TD-Learning with Linear Function Approximation}
\thanks[footnoteinfo]{This paper was not presented at any IFAC 
meeting. Corresponding author : Donghwan Lee}
\author[kaist]{Han-Dong Lim}\ead{\tt\small limaries30@kaist.ac.kr},    % Add the   
\author[kaist]{Donghwan Lee}\ead{\tt\small donghwan@kaist.ac.kr}  % (ead) as shown
\address[kaist]{Department of Electrical Engineering, KAIST, Daejeon, 34141, South Korea}  % Please supply
% \address[Donghwan Lee]{Department of Electrical Engineering, KAIST, Daejeon, 34141, South Korea}             % full addresses


\begin{abstract}

This paper analyzes multi-step temporal difference (TD)-learning algorithms within the ``deadly triad'' scenario, characterized by linear function approximation, off-policy learning, and bootstrapping. In particular, we prove that $n$-step TD-learning algorithms converge to a solution as the sampling horizon $n$ increases sufficiently. The paper is divided into two parts. In the first part, we comprehensively examine the fundamental properties of their model-based deterministic counterparts, including projected value iteration, gradient descent algorithms, which can be viewed as prototype deterministic algorithms whose analysis plays a pivotal role in understanding and developing their model-free reinforcement learning counterparts. In particular, we prove that these algorithms converge to meaningful solutions when $n$ is sufficiently large. Based on these findings, in the second part, two $n$-step TD-learning algorithms are proposed and analyzed, which can be seen as the model-free reinforcement learning counterparts of the model-based deterministic algorithms.
\end{abstract}

\end{frontmatter}

% \begin{IEEEkeywords}
% Reinforcement learning, TD-learning, machine learning, convergence, finite-time analysis
% \end{IEEEkeywords}

\section{Introduction}
Reinforcement learning (RL)~\cite{sutton1998reinforcement} seeks to find an optimal sequence of decisions in unknown systems through experiences. Recent breakthroughs showcase RL algorithms surpassing human performance in various challenging tasks~\cite{mnih2015human,lillicrap2016continuous,heess2015memory,van2016deep,bellemare2017distributional,schulman2015trust,schulman2017proximal}. This success has ignited a surge of interest in RL, both theoretically and experimentally.

Among various algorithms, temporal-difference (TD) learning~\cite{sutton1988learning} stands as a cornerstone of RL, specifically for policy evaluation. Its convergence has been extensively studied over decades~\cite{tsitsiklis1997analysis}. However, a critical challenge emerges within the ``deadly triad'' scenario, characterized by linear function approximation, off-policy learning, and bootstrapping~\cite{sutton1998reinforcement,van2018deep,chen2023target}. In such scenarios, TD-learning can diverge, leading to unreliable value estimates.

Recently, gradient temporal-difference learning (GTD) has been developed and investigated in various studies~\cite{sutton2009convergent,sutton2009fast,ghiassian2020gradient,lee2023new,lim2022backstepping}. This method addresses the deadly triad issue by employing gradient-based schemes. However, the GTD family of algorithms requires somewhat restrictive assumptions about the underlying environment, which constitutes a limitation of the method. Another approach, such as emphatic method~\cite{hallak2016generalized} or adding a regularization term~\cite{bharadwaj2020convergent}, fixes the deadly triad issue but converges to a biased solution.~\cite{zhang2021breaking} also used a regularization, which results in a biased solution. Furthermore, a target network update and a projection step are required. A comprehensive overview of off policy TD-learning algorithms can be found in~\cite{dann2014policy}.

On the other hand, TD-learning is usually implemented within the context of single-step bootstrapping based on a single transition, which is known as single-step TD-learning. These methods can be extended to include multiple time steps, a class of algorithms known as multi-step TD learning, to enhance performance. Recently, multi-step approaches~\cite{sutton1998reinforcement,tsitsiklis1997analysis,chen2021finite,mahmood2017multi,de2018multi,precup2001off,maei2010toward,van2016effective,mandal2023n,schulman2015high,carvalho2023multi}, including $n$-step TD-learning and TD($\lambda$), have become integral to the success of modern deep RL agents, significantly improving performance~\cite{schulman2015high,yuan2019novel,hessel2018rainbow,hernandez2019understanding} in various scenarios. Despite these empirical successes and the growing body of analysis on multi-step RL, to the best of the author's knowledge, the effects and theoretical underpinnings of $n$-step TD-learning have yet to be fully explored.

Motivated by the aforementioned discussions, this paper conducts an in-depth examination of the theoretical foundations necessary to understand the core principles of $n$-step TD-learning methods and their model-based counterparts, which can be viewed as prototype deterministic algorithms whose analysis plays a pivotal role in understanding and developing their model-free RL counterparts. First, we investigate the convergence conditions for $n$-step projected value iteration and present an algorithm for solving the projected $n$-step Bellman equation. We show that the projected Bellman operator becomes a contraction mapping for sufficiently large $n$, ensuring the convergence of the corresponding algorithms. We also establish a relationship between this convergence and the singularity of the matrix governing the $n$-step TD method. Next, we demonstrate that $n$-step TD methods effectively mitigate the challenges of the deadly triad when the sampling horizon, $n$, is sufficiently large. Our thorough analysis of the conditions on $n$ offering valuable insights, and we provide an interesting example why sharpening the bound might be difficult. Overall, we present the relationship between the choice of $n$ for the convergence of $n$-step projected value iteration, the singularity of the matrix, and the stability of the $n$-step TD method.

%We prove that in this case (sufficiently large $n$), the projected Bellman equation becomes a contraction mapping. This property ensures the convergence of the corresponding TD-learning algorithm towards a useful solution, which we subject to thorough analysis.

 Lastly, following the spirit of~\cite{borkar2000ode}, we prove the asymptotic convergence of $n$-step TD-learning method under both the i.i.d. and Markov observation models. The asymptotic convergence relies on the theoretical properties derived based on its model-based counterparts. We investigate the ODE counterpart of the $n$-step TD-learning, which inherits the properties of the model-based deterministic counterparts.

\section{Preliminaries}

\subsection{Notation}
The adopted notation is as follows: ${\mathbb R}$: set of real numbers; $\mathbb{R}_+$ : set of positive real numbers; ${\mathbb R}^n $: $n$-dimensional Euclidean
space; ${\mathbb R}^{n \times m}$: set of all $n \times m$ real
matrices; $A^{\top}$: transpose of matrix $A$; $A \succ 0$ ($A \prec
0$, $A\succeq 0$, and $A\preceq 0$, respectively): symmetric
positive definite (negative definite, positive semi-definite, and
negative semi-definite, respectively) matrix $A$; $I$: identity matrix with appropriate dimensions; $\lambda_{\min}(A)$ and $\lambda_{\max}(A)$ for any matrix $A$: the minimum and maximum eigenvalues of $A$; $|\mathcal{S}|$: cardinality of a finite set $\mathcal{S}$; $\left\|\cdot \right\|_{\infty}$ : infinity norm of a matrix or vector;

\subsection{Markov decision process}
A Markov decision process (MDP) is characterized by a quintuple ${\mathcal M}: =
(\mathcal{S},\mathcal{A},P,r,\gamma)$, where $\mathcal{S}$ is a finite
state-space, $\mathcal{A}$ is a finite action
space, $P(s'|s,a)$ represents the (unknown)
state transition probability from state $s$ to $s'$ given action
$a$, $r:{\mathcal S}\times {\mathcal A}\times {\mathcal S}\to
\Rset$ is the reward
function, and $\gamma \in (0,1)$ is the discount factor. In particular, if action
$a$ is selected with the current state $s$, then the state
transits to $s'$ with probability $P(s'|s,a)$ and incurs a
reward $r(s,a,s')$.  For convenience, we consider a deterministic reward function and simply write $r(s_k,a_k ,s_{k + 1}) =:r_{k+1},k \in \{ 0,1,\ldots \}$. As long as the reward function is bounded, we can assume that the reward function follows a probability distribution depending on $(s,a,s^{\prime})$.

The stochastic policy represents a probability distribution over the action space. Consider a policy $\pi:{\mathcal S} \times
{\mathcal A}\to [0,1]$ representing the probability, $\pi(a|s)$, of selecting action $a$ at the current state $s$, $P^\pi$ denotes the state transition probability matrix under policy $\pi$, and $d^{\pi}:{\mathcal S} \to {\mathbb R}$ denotes the stationary probability distribution of the state $s\in {\mathcal S}$ under $\pi$. We also define
$R^\pi(s)$ as the expected reward given the policy $\pi$ and the current state $s$. The infinite-horizon discounted value function with policy $\pi$ is $v^{\pi}(s):={\mathbb E} \left[ \left. \sum_{k = 0}^\infty {\gamma
^k r(s_k,a_k,s_{k+1})} \right|s_0=s \right]$, where ${\mathbb E}$ stands for the expectation taken with respect to the state-action trajectories under $\pi$. Given pre-selected basis (or feature) functions $\phi_1,\ldots,\phi_m:{\mathcal S}\to {\mathbb R}$, the matrix, $\Phi \in {\mathbb R}^{|{\mathcal S}| \times m}$, called the feature matrix, is defined as a matrix whose $s$-th row vector is $\phi(s):=\begin{bmatrix} \phi_1(s) &\cdots & \phi_m(s) \end{bmatrix}$. Throughout the paper, we assume that $\Phi \in {\mathbb R}^{|{\mathcal S}| \times m}$ is a full column rank matrix, which can be guaranteed by using Gaussian basis or Fourier feature functions. The policy evaluation problem is the problem of estimating $v^{\pi}$ given a policy $\pi$. In this paper, we will denote $V^{\pi}\in\mathbb{R}^{|\mathcal{S}|}$ to be a vector representation of the value function, i.e., the $s$-th element of $V^{\pi}$ corresponds to $v^{\pi}(s)$.

\begin{definition}[Policy evaluation problem]\label{def:1}
In this paper, the policy evaluation problem is defined as finding the least-square solution
\begin{align*}
\theta _*^\infty  = \arg {\min _{\theta  \in {\mathbb R}^m}}f(\theta ),\quad f(\theta ) = \frac{1}{2}\left\| {{V^\pi } - \Phi \theta } \right\|_{{D^\beta }}^2
\end{align*}
where
% \begin{align*}
% f(\theta ) = \frac{1}{2}\left\| {{V^\pi } - \Phi \theta } \right\|_{{D^\beta }}^2
% \end{align*}
% and
$
{V^\pi }: = \sum\limits_{k = 0}^\infty  {{\gamma ^k}{{({P^\pi })}^k}{R^\pi }}
$
is the true value function, $R^\pi \in {\mathbb R}^{|{\mathcal S}|}$ is a vector enumerating all $R^\pi(s), s\in {\mathcal S}$, $D^{\beta}$ is a diagonal matrix with positive diagonal elements $d^{\beta}(s),s\in {\mathcal S}$, and $\|x\|_D:=\sqrt{x^T Dx}$ for any positive-definite $D$. Here, $d^{\beta}$ can be any state visit distribution under the behavior policy $\beta$ such that $d^{\beta}(s)>0,\forall s\in {\mathcal S}$. The solution can be written as
\begin{align}
\Phi \theta _*^\infty  = \Pi {V^\pi }.\label{eq:optimal-solution}
\end{align}
 $\Pi$ is the projection onto the range space of $\Phi$, denoted by $\mathcal{R}(\Phi)$: $\Pi(x):=\argmin_{x'\in \mathcal{R}(\Phi)}
\|x-x'\|_{D^{\beta}}^2$. The projection can be performed by the matrix
multiplication: we write $\Pi(x):=\Pi x$, where $\Pi:=\Phi(\Phi^T
D^{\beta} \Phi)^{-1}\Phi^T D^{\beta} \in \mathbb{R}^{|\mathcal{S}|\times |\mathcal{S}|}$.

\end{definition}

\subsection{Review of GTD algorithm}\label{sec:gradient1}
In this section, we briefly review the gradient temporal difference (GTD) learning developed in~\cite{sutton2009convergent}, which tries to solve the policy evaluation problem. Roughly speaking, the goal of the policy evaluation is to find the weight vector $\theta$ such that $\Phi \theta$ approximates the true value function $V^{\pi}$. This is typically done by minimizing the so-called {\em mean-square projected Bellman error} loss function~\cite{sutton2009convergent,sutton2009fast}
\begin{align}
&\min_{\theta\in {\mathbb R}^q} {\rm MSPBE}(\theta):= \frac{1}{2}\|
\Pi (R^{\pi} + \gamma P^{\pi} \Phi \theta)-\Phi \theta \|_{D^{\beta}}^2.\label{eq:4}
\end{align}
Note that minimizing the objective means minimizing the error of the projected Bellman equation (PBE) $\Phi \theta  = \Pi (R^\pi   + \gamma P^\pi  \Phi \theta )$ with respect to $ \|\cdot\|_{D^\beta}$. Moreover, note that in the objective of~\eqref{eq:4}, $d^{\beta}$ depends on the behavior policy, $\beta$, while $P^{\pi}$ and $R^{\pi}$ depend on the target policy, $\pi$, that we want to evaluate. This structure allows us to obtain an off-policy learning algorithm through the importance sampling~\cite{precup2001off} or sub-sampling techniques~\cite{sutton2009convergent}.

A common assumption in proving the convergence of GTD~\cite{sutton2009convergent,sutton2009fast,ghiassian2020gradient,lee2023new} is the following assumption:
\begin{assumption}\label{assumption:2}
$\Phi ^{\top} D^{\beta} (\gamma P^\pi   - I)\Phi $ is nonsingular.
\end{assumption}
Please note that \cref{assumption:2} always holds when $\beta = \pi$, while it is in general not true. It will be clear in further section how this assumption can be relaxed using $n$-step methods. A sufficiently large choice of $n$ can relax this assumption. Moreover, the value of $n$ is chosen to be finite, which clearly differs with the Monte-Carlo setting where $n=\infty$.

Some properties related to~\eqref{eq:4} are summarized below for convenience and completeness.
\begin{lemma}[Lemma 3 in~\cite{lee2023new}]\label{lemma:2}
Under Assumption~\ref{assumption:2}, the following statements hold true:
\begin{enumerate}
\item A solution of~\eqref{eq:4} exists, and is unique.

\item The solution of~\eqref{eq:4} is given by
\begin{align}
\theta ^*:=-(\Phi ^{\top} D^{\beta} (\gamma P^\pi   - I)\Phi )^{ - 1} \Phi ^{\top} D^{\beta} R^\pi.\label{eq:theta-star}
\end{align}
\end{enumerate}
\end{lemma}

% Based on the objective function in~(\ref{eq:4}),~\cite{sutton2009fast} developed GTD2.

\section{Multi-step projected Bellman operator}\label{sec:dynamic-programming}

Let us consider the $n$-step Bellman operator~\cite{sutton1988learning}
\begin{align*}
{T^n}(x):=& (I+ \gamma {P^\pi } +  \cdots  + {\gamma ^{n - 1}}{({P^\pi })^{n - 1}}){R^\pi } + {\gamma ^n}{({P^\pi })^n}x.
\end{align*}
Then, the corresponding projected $n$-step Bellman operator ($n$-PBO) is given by $\Pi {T^n}$. Based on this, the corresponding $n$-step projected value iteration ($n$-PVI) is given by
\begin{align}
\Phi {\theta _{k + 1}} = \Pi {T^n}(\Phi {\theta _k}),\quad k \in \{ 0,1, \ldots \} ,\quad {\theta _0} \in {\mathbb R}^m\label{eq:projected-VI}
\end{align}

Note that at each iteration $k$, $\theta _{k + 1}$ can be uniquely determined given $\theta _k$ because $\Pi {T^n}(\Phi {\theta _k})$ belongs to the image of $\Phi$, and the unique solution solves $\Phi \theta  = \Pi {T^n}(\Phi {\theta _k}),$ and is given by
\begin{align}
{\theta _{k + 1}} = {({\Phi ^{\top}}D^\beta\Phi )^{ - 1}}{\Phi ^{\top}}D^\beta {T^n}(\Phi {\theta _k}).\label{eq:projected-VI2}
\end{align}
Moreover, it is important to note that $\Pi \in \mathbb{R}^{|\mathcal{S}|\times |\mathcal{S}|}$ is a projection onto the column space of the feature matrix $\Phi$ with respect to the weighted norm ${\left\|  \cdot  \right\|_{{D^\beta }}}$, and satisfies the nonexpansive mapping property ${\left\| {\Pi x - \Pi y} \right\|_{{D^\beta }}} \le {\left\| {x - y} \right\|_{{D^\beta }}}$ with respect to ${\left\|  \cdot  \right\|_{{D^\beta }}}$.
On the other hand, for the Bellman operator $T^n$, we can consider the two cases depending on the behavior policy and target policy:
\begin{enumerate}
\item on-policy case: the behavior policy and target policy are identical, i.e., $\beta = \pi$,
\item off-policy case : the behavior policy and target policy are different, i.e., $\beta \neq \pi$.
\end{enumerate}

In the on-policy case $\beta = \pi$, it can be easily proved that $T^n$ is a contraction mapping with respect to the norm ${\left\|  \cdot  \right\|_{{D^\beta }}}$ with the contraction factor $\gamma^n$.
\begin{lemma}
If $\beta=\pi$, the mapping $T^n$ satisfies
\begin{align*}
{\left\| {{T^n}(x) - {T^n}(y)} \right\|_{{D^\pi }}} \le {\gamma ^n}{\left\| {x - y} \right\|_{{D^\pi }}},\quad \forall x,y \in {\mathbb R}^{|S|}.
\end{align*}
\end{lemma}
\begin{proof}
The proof can be easily done by following the main ideas of~\cite[Lemma~4]{tsitsiklis1997analysis}, and omitted here for brevity.
\end{proof}

Therefore, $n$-PBO$, \Pi {T^n}$, is also a contraction with the factor $\gamma^n$.
\begin{lemma}
If $\beta=\pi$, the mapping $\Pi T^n$ satisfies
\begin{align*}
{\left\| {{\Pi T^n}(x) - {\Pi T^n}(y)} \right\|_{{D^\pi }}} \le {\gamma ^n}{\left\| {x - y} \right\|_{{D^\pi }}},\quad \forall x,y \in {\mathbb R}^{|S|}.
\end{align*}
\end{lemma}
The above result implies that $\Pi T^n$ is a contraction.
In conclusion, by Banach fixed point theorem, $n$-PVI in~\eqref{eq:projected-VI} converges to its unique fixed point because $n$-PBO $\Pi {T^n}$ is a contraction with respect to $\left\|  \cdot  \right\|_{D^\beta }$.

 On the other hand, in the off-policy case $\beta \neq \pi$, $T^n$ is no more a contraction mapping with respect to ${\left\|  \cdot  \right\|_{{D^\beta }}}$, and so is $\Pi T^n$. Therefore, $n$-PVI in~\eqref{eq:projected-VI} may not converge in some cases.
However, in this paper, we will prove that for a sufficiently large $n$, $\Pi T^n$ becomes a contraction with respect to some norm. To proceed further, some necessary notions are defined.
\begin{definition}\label{def:nPBE}
A solution of the $n$-PBE, $\theta_*^n$, if exists, is defined as a vector satisfying
\begin{align}
\Phi \theta_*^n  = \Pi {T^n}(\Phi {\theta_*^n}).\label{eq:1}
\end{align}
\end{definition}
Using the optimal Bellman equation $\Phi \theta_*^n  = \Pi {T^n}(\Phi {\theta_*^n})$,~\eqref{eq:projected-VI2} can be rewritten by
\begin{align}
{\theta _{k + 1}} - \theta _*^n = {{({\Phi ^{\top}}D^{\beta}\Phi )}^{ - 1}}{\Phi ^{\top}}D^{\beta}{\gamma ^n}{{({P^\pi })}^n}\Phi ({\theta _k} - \theta _*^n),\label{eq:projected-VI3}
\end{align}
which is a discrete-time linear time-invariant system~\cite{chen1995linear}. This implies that the convergence of~\eqref{eq:projected-VI3} is fully characterized by the Schur stability of the matrix 
\begin{align}
    A :=- (\Phi ^{\top}D^{\beta}\Phi )^{ - 1}\Phi ^{\top}D^{\beta}\gamma^n{({P^\pi })}^n \label{def:A}
\end{align}
 Moreover, one can conjecture that the existence and uniqueness of the solution to $n$-PBE in~\eqref{def:nPBE} is also related to the Schur stability of $A$ as well. Indeed, one can prove that the Schur stability and contraction property of $n$-PBO $\Pi {T^n}$ are equivalent. 
\begin{theorem}\label{thm:Schur-eq}
The matrix $A$ defined in~(\ref{def:A}) is Schur if and only if $\Pi T^n$ is a contraction.
\end{theorem}
\begin{proof}
Noting that $n$-PVI is equivalently written by~\eqref{eq:projected-VI3}, one can easily prove that the convergence of $n$-PVI is equivalent to that of the linear system in~\eqref{eq:projected-VI3}. Moreover, from the standard linear system theory,~\eqref{eq:projected-VI3} converges to a unique point if and only if $A$ is Schur. Then, since $\Pi T^n$ is an affine mapping, one arrives at the desired conclusion using~\cref{lemma:Schur-eq2}. 
\end{proof}

\begin{remark}
~\cref{thm:Schur-eq} implies the equivalence between the matrix $A$ being Schur and $\Pi T^n$ being a contraction.~\cref{lemma:Schur-eq2} ensures the equivalence of $\Pi T^n$ being a contraction and convergence of $n$-PVI. Therefore, we can conclude that $A$ is Schur if and only if $n$-PVI converges.    
\end{remark}

In the next theorem, we establish a connection between the contraction property of $\Pi {T^n}$ and the nonsingularity of ${\Phi ^T}{D^\beta }(I - {\gamma ^n}{({P^\pi })^n})\Phi$, which plays an important role throughout the paper.
\begin{lemma}[Corollary 5.6.16 in~\cite{horn2012matrix}]\label{lemma:7}
If $M \in {\mathbb R}^{n\times n}$ satisfies $\|M \|<1$ for some matrix norm $\|\cdot \|$, then $I-M$ is nonsingular, and
\[\left\| (1 - M)^{ - 1} \right\| \le \frac{1}{1 - \left\| M \right\|}\]
\end{lemma}

\begin{theorem}\label{lemma:nonsingular}
If $A$ is Schur, then ${\Phi ^{\top}}{D^\beta }\Phi (I - A) = {\Phi ^{\top}}{D^\beta }(I - {\gamma ^n}{(P^\pi)^n})\Phi $ is nonsingular.
\end{theorem}
\begin{proof}
If $A$ is Schur, then $\rho(A)<1$, where $\rho$ is the spectral radius. Then, following similar arguments as in the proof of~\cref{lemma:Schur-eq}, one can prove that there exists a matrix norm $\|\cdot \|$ such that $\left\| A \right\| < 1$. Then, by~\cref{lemma:7}, we have that $I - A = I - {({\Phi ^T}{D^\beta }\Phi )^{ - 1}}{\Phi ^{\top}}{D^\beta }{\gamma ^n}{(P^\pi)^n}\Phi$ is nonsingular. Equivalently, ${\Phi ^{\top}}{D^\beta }\Phi (I - A) = {\Phi ^T}{D^\beta }(I - {\gamma ^n}{(P^\pi)^n})\Phi $ is nonsingular.
\end{proof}

\begin{remark}
    With sufficiently large $n$, we can now relax the Assumption~\ref{assumption:2}. However, the nonsingularity of ${\Phi ^{\top}}{D^\beta }(I - {\gamma^n}{({P^\pi })^n})\Phi$ does not imply that $n$-PBO, $\Pi {T^n}$, is a contraction mapping with respect to some $\| \cdot \|$. To support this argument, we provide a counter example below.
\end{remark}

\begin{example}
Let us consider a MDP with two states and a single action:
 \begin{align*}
     \Phi:=\begin{bmatrix}
         1\\
         3
     \end{bmatrix},\quad D= \begin{bmatrix}
         0.5 & 0\\
         0 & 0.5
     \end{bmatrix},\quad P^{\pi}=\begin{bmatrix}
         0 & 1\\
         0 & 1
     \end{bmatrix},
 \end{align*}
 and $n=1$ and $\gamma=0.99$. Then, ${\Phi ^{\top}}{D^\beta }(I - {\gamma ^n}{({P^\pi })^n})\Phi$ is non-singular but $\Pi T^n$ is not a contraction mapping, i.e., the spectral radius of $\gamma\Pi (P^{\pi})^n$ is bigger than one.    
\end{example}

In summary, we have proved until now that
\begin{align*}
     &(\text{$\Pi T^n$ is contraction}\iff \text{$A$ is Schur})\\
     \Rightarrow& \text{${\Phi ^{\top}}{D^\beta }(I - {\gamma ^n}{({P^\pi })^n})\Phi$ is non-singular}, 
\end{align*}
while the converse does not necessarily holds.

Next, we establish and summarize several results such as a sufficient condition on $n$ such that $A$ becomes Schur.
\begin{lemma}\label{lem:A_Schur}
The following statements hold true:
\begin{enumerate}
\item There exists a positive integer $n_1^*$ such that $A$ is Schur:
\begin{align*}
    n_1^* \leq \left\lceil \ln\left( \frac{1}{||(\Phi^{\top}D^{\beta}\Phi)^{-1}\Phi^{\top}D^{\beta}||_{\infty}||\Phi||_{\infty}}\right)/\ln(\gamma) \right \rceil,
\end{align*}
\item Suppose that $n \geq n^*_1$ so that $A$ is Schur. Then, ${\Phi ^{\top}}{D^\beta }(I - {\gamma ^n}{({P^\pi })^n})\Phi$ is nonsingular. Moreover, $n$-PVI in~\eqref{eq:projected-VI} converges to the unique fixed point $\theta_*^n$ and satisfies
\begin{align*}
{\left\| {\theta _k} - \theta _*^n \right\|_\infty} \le {\left\| A \right\|^k_{\infty}}{\left\| \theta _0 - \theta _*^n \right\|_\infty }.
\end{align*}

Furthermore, the unique fixed point of $\Pi {T^n}$, denoted by $\theta _*^n$, is given by
\begin{align}
\theta_*^n =& {\left[ {{\Phi ^{\top}}{D^\beta }(I - {\gamma ^n}{(P^\pi)^n})\Phi } \right]^{ - 1}}{\Phi^{\top}}{D^\beta }\nonumber\\
&\times (R^\pi + \gamma P^\pi R^\pi +  \cdots  + \gamma ^{n - 1}{(P^\pi)^{n - 1}}{R^\pi }).\label{eq:solution1}
\end{align}

\item There always exists a positive integer $n^*_2 < \infty$ such that $\Pi {T^n}$ is a contraction with respect to ${\left\|  \cdot  \right\|_\infty}$. Moreover, we have ${n^*_2} \le \left\lceil {\frac{\ln (\left\| \Pi  \right\|_\infty ^{ - 1})}{\ln (\gamma )}} \right\rceil+1$, where $\left\lceil  \cdot  \right\rceil$ stands for the ceiling function.
\end{enumerate}
\end{lemma}
\begin{proof}
For the first item, let us bound $||A||_{\infty}$:
\begin{align*}
    ||A||_{\infty} \leq &\gamma^n ||(\Phi^{\top}D^{\beta}\Phi)^{-1}\Phi^{\top}D^{\beta}||_{\infty}||(P^{\pi})^n||_{\infty}||\Phi||_{\infty}\\
    \leq & \gamma^n||(\Phi^{\top}D^{\beta}\Phi)^{-1}\Phi^{\top}D^{\beta}||_{\infty}||\Phi||_{\infty},
\end{align*}
where the second inequality follows from the fact that $||P^{\pi}||_{\infty}=1$. Hence, for any integer $n_1^*$ such that
\begin{align*}
    n_1^* > \ln\left( \frac{1}{||(\Phi^{\top}D^{\beta}\Phi)^{-1}\Phi^{\top}D^{\beta}||_{\infty}||\Phi||_{\infty}}\right)/\ln(\gamma),
\end{align*}
we have $||A||_{\infty} < 1$ for any $n > n^*_1$. 
By using Gelfand's formula, $||A||_{\infty} < 1$ implies that $A$ is Schur. 
This completes the proof of the first item.
The first statement in the second item can be directly proved from~\cref{lemma:nonsingular} and~\cref{lem:A_Schur}. The second statement is due to
\begin{align*}
    ||\theta_k-\theta^n_*||_{\infty} & 
    \leq || A||_{\infty} ||\theta_{k-1}-\theta^n_*||_{\infty}
    \leq  ||A||_{\infty}^k||\theta_{0}-\theta^n_*||_{\infty}.
\end{align*}
Since $||A||_{\infty}< 1$ when $n \geq n^*_1$ from Lemma~\ref{lem:A_Schur}, we have $||A||_{\infty}^k\to 0 $ as $k\to\infty$. Therefore, $n$-PVI in~\eqref{eq:projected-VI} converges to the unique fixed point. 
The third statement in the second item is derived by using the nonsingularity of ${\Phi ^{\top}}D^\beta(I - {\gamma ^n}{({P^\pi })^n})\Phi$. For the fourth item, noting that
\begin{align*}
\left\| \Pi T^n (x) - \Pi T^n(y) \right\|_\infty % =& {\gamma ^n}{\left\| {\Pi ((P^\pi)^n x - (P^\pi)^n}y) \right\|_\infty }\\
% \le& {\gamma ^n}{\left\| \Pi  \right\|_\infty }{\left\| (P^\pi)^n(x - y) \right\|_\infty }\\
\le& {\gamma ^n}{\left\| \Pi  \right\|_\infty }{\left\| (P^\pi )^n \right\|_\infty }{\left\| x - y \right\|_\infty }\\
=& {\gamma ^n}{\left\| \Pi  \right\|_\infty }{\left\| x - y \right\|_\infty },
\end{align*}
for a sufficiently large $n'$, we have $\gamma ^n \left\| \Pi  \right\|_\infty  < 1$ for all $n \ge n'$, which implies that $\Pi T^n$ is a contraction mapping with respect to ${\left\|  \cdot  \right\|_\infty}$ for all $n \ge n'$. Therefore, there exists an $n^*_2 < \infty$. Moreover, ${\gamma ^n}{\left\| \Pi  \right\|_\infty } < 1$ is equivalent to $n\ln (\gamma ) < \ln (\left\| \Pi  \right\|_\infty ^{ - 1})$, or equivalently, $n > \frac{{\ln (\left\| \Pi  \right\|_\infty ^{ - 1})}}{{\ln (\gamma )}}$. Taking the ceiling function on the left-hand side, a sufficient condition is $n \ge \left\lceil {\frac{{\ln (\left\| \Pi  \right\|_\infty ^{ - 1})}}{{\ln (\gamma )}}} \right\rceil +1$. Therefore, one concludes that ${n^*_2} \le \left\lceil {\frac{\ln (\left\| \Pi  \right\|_\infty ^{ - 1})}{\ln (\gamma )}} \right\rceil  + 1$.
\end{proof}
The results in~\cref{lem:A_Schur} tell us that the solution $\theta_*^n$ of $n$-PBE varies according to $n$. Therefore, a question that naturally arises here is regarding the relevance of $\theta_*^n$ in comparison to the true optimal solution $\theta_*^\infty$ and the true value function $V^\pi$. In the following theorem, bounds on the errors among the different solutions are given.
\begin{theorem}
 For all $n \geq n^*_2$, $\Pi {T^n}$ is a contraction with respect to ${\left\|  \cdot  \right\|_\infty}$.
Then, we have
\begin{align}
{\left\| \Phi \theta^n_* - V^\pi  \right\|_\infty } \le \frac{1}{1 - \gamma ^n \left\| \Pi  \right\|_\infty}{\left\| {\Pi V^\pi  - V^\pi} \right\|_\infty }\label{eq:2}
\end{align}
and
\begin{align}
{\left\| {\Phi \theta^n_* - \Phi {\theta_*^\infty}} \right\|_\infty } \le \frac{{{\gamma ^n}{{\left\| \Pi  \right\|}_\infty }}}{{1 - {\gamma ^n}{{\left\| \Pi  \right\|}_\infty }}}{\left\| {\Pi {V^\pi } - {V^\pi }} \right\|_\infty }\label{eq:3}.
\end{align}
\end{theorem}
\begin{proof}
By hypothesis, $\Pi {T^n}$ is a contraction, which means that there exists a unique solution $\theta ^n_*$ satisfying $n$-PBE, which can be rewritten by
\begin{align*}
\Pi {T^n}(\Phi \theta^n_*) - {V^\pi } = \Phi \theta^n_* - {V^\pi }.
\end{align*}
The left-hand side can be written as
\begin{align*}
&\Phi \theta^n_* - {V^\pi }\nonumber\\
=&\Pi {T^n}(\Phi \theta^n_*) - {V^\pi }\nonumber\\
% =& \Pi {T^n}(\Phi \theta^n_*) - {V^\pi } - \Pi {V^\pi } + \Pi {V^\pi }\nonumber\\
=& \Pi ({T^n}(\Phi \theta^n_*) - {V^\pi }) + \Pi {V^\pi } - {V^\pi }\nonumber\\
=& \Pi \left( V^\pi  - \sum\limits_{k = n}^\infty  \gamma ^k(P^\pi)^k R^\pi  + {\gamma ^n}{(P^\pi)^n}\Phi \theta^n_* - V^\pi \right)\nonumber\\
& + \Pi V^\pi  - V^\pi \nonumber\\
=& \Pi ({\gamma ^n}{(P^\pi)^n}\Phi \theta^n_* - {\gamma ^n}{({P^\pi })^n}{V^\pi }) + \Pi {V^\pi } - {V^\pi }.
\end{align*}
Next, taking the norm ${\left\|  \cdot  \right\|_\infty}$ on both sides of the above inequality leads to
\begin{align*}
&{\left\| {\Phi \theta^n_* - {V^\pi }} \right\|_\infty }\\
% =& {\left\| {\Pi ({\gamma ^n}{{({P^\pi })}^n}\Phi \theta^n_* - {\gamma ^n}{{({P^\pi })}^n}{V^\pi }) + \Pi {V^\pi } - {V^\pi }} \right\|_\infty }\\
\le& {\left\| {{\gamma ^n}\Pi {{({P^\pi })}^n}(\Phi \theta^n_* -V^{\pi} )} \right\|_\infty } + {\left\| {\Pi {V^\pi } - {V^\pi }} \right\|_\infty }\\
% \le& {\gamma ^n}{\left\| \Pi  \right\|_\infty }{\left\| {{{({P^\pi })}^n}} \right\|_\infty }{\left\| {\Phi \theta^n_* - {V^\pi }} \right\|_\infty }
% % &\times {\left\| {{{({P^\pi })}^n}\Phi \theta^n_* - {{({P^\pi })}^n}{V^\pi }} \right\|_\infty }\\
%  + {\left\| {\Pi {V^\pi } - {V^\pi }} \right\|_\infty }\\
\le & {\gamma ^n}{\left\| \Pi  \right\|_\infty }{\left\| {\Phi \theta^n_* - {V^\pi }} \right\|_\infty } + {\left\| {\Pi {V^\pi } - {V^\pi }} \right\|_\infty },
\end{align*}
which yields
\begin{align*}
(1 - {\gamma ^n}{\left\| \Pi  \right\|_\infty }){\left\| {\Phi \theta^n_* - {V^\pi }} \right\|_\infty } \le {\left\| {\Pi {V^\pi } - {V^\pi }} \right\|_\infty }.
\end{align*}
By hypothesis, $n \geq n_2^*$ implies that $1 - {\gamma ^n}{\left\| \Pi  \right\|_\infty } > 0$ holds. Therefore, the last inequality leads to~\eqref{eq:2}.

Similarly, combining~\eqref{eq:1} and~\eqref{eq:optimal-solution} yields
\begin{align*}
&\Phi (\theta _*^n - \theta _*^\infty )\\
% =& \Pi ({T^n}(\Phi \theta _*^n) - {V^\pi })\\
=& \Pi \left( {{V^\pi } - \sum\limits_{k = n}^\infty  {{\gamma ^k}{{({P^\pi })}^k}{R^\pi }}  + {\gamma ^n}{{({P^\pi })}^n}\Phi \theta _*^n - {V^\pi }} \right)\\
% =& \Pi \left( {{\gamma ^n}{{({P^\pi })}^n}\Phi \theta _*^n - \sum\limits_{k = n}^\infty  {{\gamma ^k}{{({P^\pi })}^k}{R^\pi }} } \right)\\
 =& {\gamma ^n}\Pi ({({P^\pi })^n}\Phi \theta _*^n - {({P^\pi })^n}{V^\pi }).
\end{align*}
Now, taking the norm ${\left\|  \cdot  \right\|_\infty}$ on both sides of the above inequality leads to
\begin{align*}
{\left\| {\Phi (\theta _*^n - \theta _*^\infty )} \right\|_\infty }
=& {\left\| {\Pi {{({P^\pi })}^n}\Phi \theta _*^n - \Pi {{({P^\pi })}^n}{V^\pi }} \right\|_\infty }\\
\le& {\gamma ^n}{\left\| \Pi  \right\|_\infty }{\left\| {\Phi \theta _*^n - {V^\pi }} \right\|_\infty }\\
\le& {\gamma ^n}{\left\| \Pi  \right\|_\infty }\frac{1}{{1 - {\gamma ^n}{{\left\| \Pi  \right\|}_\infty }}}{\left\| {\Pi {V^\pi } - {V^\pi }} \right\|_\infty },
\end{align*}
where the second inequality comes from~\eqref{eq:2}.
\end{proof}
The inequality in~\eqref{eq:2} provides an error bound between $\Phi \theta^n_*$ and the true value function ${V^\pi }$.
Moreover,~\eqref{eq:3} gives an error bound between $\Phi \theta^n_*$ and the true optimal solution $\Phi \theta^\infty_*$.
One can observe that the second bound vanishes as $n \to \infty$, while the first bound remains nonzero. This is because there remain fundamental errors due to the linear function approximation, which can become zero when the feature matrix $\Phi$ is chosen appropriately.

Until now, we have studied properties of $n$-PBO and the corresponding $n$-PVI in~\eqref{eq:1}. These properties play important roles for the development of the corresponding RL algorithms.  Nonetheless, to implement the algorithm in~(\ref{eq:8}), we require a matrix inversion, which is often not possible and cannot be implemented when we only observe stochastic samples. In the next sections, we will study some alternative approaches based on gradients to solve the policy evaluation problem.

\section{Deterministic algorithm}\label{sec:system}
In this section, we will consider another class of model-based iterative algorithms motivated by the methods for solving general linear equations~\cite{kelley1995iterative}.
In particular, let us first consider the $n$-PBE in~\eqref{eq:1} again
\begin{align*}
{\Phi ^{\top}}{D^\beta }{T^n}(\Phi \theta ) = {\Phi ^{\top}}{D^\beta }\Phi \theta,
\end{align*}
which can be written as the following linear equation form:
\begin{align*}
&\underbrace {\,{\Phi ^{\top}}{D^\beta }({R^\pi } + \gamma {P^\pi }{R^\pi } +  \cdots  + {\gamma ^{n - 1}}{{({P^\pi })}^{n - 1}}{R^\pi })}_{=:b}\\
=& \underbrace {{\Phi ^{\top}}{D^\beta }(I - {\gamma ^n}{{({P^\pi })}^n})\Phi }_{=:N}\theta.
\end{align*}
We consider a Richardson type iteration~\cite{kelley1995iterative} of the form
\begin{align}
{\theta _{k + 1}} = {\theta _k} + \alpha {\Phi ^{\top}}D^\beta({T^n}(\Phi {\theta _k}) - \Phi {\theta _k}),\label{eq:system1}
\end{align}
where $\alpha>0$ is the step-size. 
Combining~\eqref{eq:system1} and the fixed point equation in~\eqref{eq:1}, it follows that
\begin{align}
\theta _{k + 1} - \theta _*^n = (I -\alpha N ) \theta _*^n,\label{eq:system2}
\end{align}
which is a discrete-time linear time-invariant system~\cite{chen1995linear}. Therefore, the convergence of~\eqref{eq:system1} is equivalent to the Schur stability of $I-\alpha N$. Moreover, note that the above update does not involve a matrix inversion, compared to the update in~(\ref{eq:projected-VI3}), and it naturally extends to the TD-learning allowing the sampling scheme, which will be clear in the subsequent section.

We can prove that the iterate $\theta_k$ converges to $\theta_*^n$ for a sufficiently large $n$ and sufficiently small $\alpha$.

\begin{theorem}\label{thm:nd}
There exists a positive integer $n_3^*< \infty$ such that ${{\Phi ^{\top}}D^\beta ({\gamma ^{n}}({P^\pi })^{n}-I)\Phi }$ becomes negative definite and Hurwitz. Moreover, $n^*_3\leq n_{\mathrm{th}}$ where
\[ n_{\mathrm{th}}  =   \left\lceil\frac{\ln\left(\max\left\{\frac{d_{\min}\lambda_{\min}(\Phi^{\top}\Phi)}{\phi_{\max}^2} , \frac{d_{\min}\lambda_{\min}(\Phi^{\top}\Phi)}{{{d_{\max }}{\lambda _{\max }}({\Phi ^{\top}}\Phi )}}\frac{1}{{\sqrt {|\mathcal{S}|} }}   \right\}\right)}{\ln(\gamma)} \right\rceil , \] 
where $d_{\min}=\min_{s\in\mathcal{S}}d^{\beta}(s),d_{\max}=\max_{s\in\mathcal{S}}d^{\beta}(s)\;$, and $\phi_{\max}=\max_{s\in\mathcal{S}}\left\| \phi(s)\right\|_2^2$. Furthermore, it becomes negative definite for all $n\geq n_{\mathrm{th}} $.
\end{theorem}
\begin{proof}
Since
\begin{align*}
     & \Phi^{\top}D^{\beta}(\gamma^n(P^{\pi})^n)\Phi-\Phi^{\top}D^{\beta}\Phi \\
    \prec &  \Phi^{\top}D^{\beta}(\gamma^n(P^{\pi})^n)\Phi -d_{\min} \lambda_{\min}(\Phi^{\top}\Phi) I ,
\end{align*}
it is enough to show that
\begin{align}
 x^{\top}\left( \Phi^{\top}D^{\beta}(\gamma^n(P^{\pi})^n)\Phi \right) x \leq d_{\min} \lambda_{\min}(\Phi^{\top}\Phi) ||x||^2_2 . \label{ineq:thm4-nd-main}
\end{align}
for $x\in \mathbb{R}^m$ except at the origin. There are two approaches in bounding $ x^{\top}\left( \Phi^{\top}D^{\beta}(\gamma^n(P^{\pi})^n)\Phi \right) x $. The first is
\begin{align*}
    &x^{\top}\Phi^{\top}D^{\beta}(\gamma^n(P^{\pi})^n)\Phi x\\
    =&  \gamma^n \sum_{s\in\mathcal{S}}d^{\beta}(s) \sum_{s^{\prime}\in\mathcal{S}} [(P^{\pi})^n]_{ss^{\prime}}x^{\top}\phi(s)\phi(s^{\prime})^{\top}x\\
    \leq & \gamma^n \phi_{\max}^2 \left\| x \right\|^2_2.
\end{align*}
 Therefore, we require
\begin{align}
     n \geq \frac{\ln\left( \frac{d_{\min}\lambda_{\min}(\Phi^{\top}\Phi)}{\phi_{\max}^2}\right)}{\ln(\gamma)}. \label{ineq:thm4-nd-1}
\end{align}
Meanwhile, another approach to satisfy~(\ref{ineq:thm4-nd-main}) is
\begin{align}
&{\gamma ^n}{\lambda _{\max }}({\Phi ^{\top}}{D^\beta }{({P^\pi })^n}\Phi  + {\Phi ^{\top}}{({({P^\pi })^{\top}})^n}{D^\beta }\Phi ) \nonumber\\
<& 2d_{\min}\lambda_{\min}(\Phi^{\top}\Phi).\label{eq:7}
\end{align}

The left-hand side can be bounded as
\begin{align*}
&{\lambda _{\max }}({\Phi ^{\top}}{D^\beta }{({P^\pi })^n}\Phi  + {\Phi ^{\top}}{({({P^\pi })^{\top}})^n}{D^\beta }\Phi )\\
% \le& |{\lambda _{\max }}({\Phi ^{\top}}{D^\beta }{({P^\pi })^n}\Phi  + {\Phi ^{\top}}{({({P^\pi })^{\top}})^n}{D^\beta }\Phi )|\\
% \le& |{\lambda _{\max }}({D^\beta }{({P^\pi })^n} + {({({P^\pi })^{\top}})^n}D){\lambda _{\max }}({\Phi ^T}\Phi )|\\
\le& {\left\| {{D^\beta }{{({P^\pi })}^n} + {{({{({P^\pi })}^{\top}})}^n}{D^\beta }} \right\|_2}{\lambda _{\max }}({\Phi ^{\top}}\Phi )\\
% \le& \sqrt {|S|} {\left\| {{D^\beta }{{({P^\pi })}^n} + {{({{({P^\pi })}^{\top}})}^n}{D^\beta }} \right\|_\infty }{\lambda _{\max }}({\Phi ^{\top}}\Phi )\\
% \le& 2\sqrt {|S|} {\left\| {{D^\beta }{{({P^\pi })}^n}} \right\|_\infty }{\lambda _{\max }}({\Phi ^{\top}}\Phi )\\
\le& 2{d_{\max }}\sqrt {|S|} {\lambda _{\max }}({\Phi ^{\top}}\Phi ).
\end{align*}
Therefore, a sufficient condition for~\eqref{eq:7} is
\begin{align*}
 n > \frac{{\ln \left( \frac{{{d_{\min }}{\lambda _{\min }}({\Phi ^{\top}}\Phi )}}{{{d_{\max }}{\lambda _{\max }}({\Phi ^{\top}}\Phi )}}\frac{1}{{\sqrt {|S|} }} \right)}}{{\ln (\gamma )}}.   
\end{align*}
By combining the result in~(\ref{ineq:thm4-nd-1}) and applying the ceiling function, we obtain the desired conclusion.
\end{proof}

\begin{remark}\label{remark1}
    The inequality between two quantities $\frac{d_{\min}\lambda_{\min}(\Phi^{\top}\Phi)}{\phi_{\max}^2}$ and $\frac{d_{\min}\lambda_{\min}(\Phi^{\top}\Phi)}{{{d_{\max }}{\lambda _{\max }}({\Phi ^{\top}}\Phi )}}\frac{1}{{\sqrt {|S|}}}  $ in Theorem~\ref{thm:nd} does not hold in general. An example is given in Appendix~\ref{example:compare}. This enhances the bound in~\cite{carvalho2023multi}, where only the quantity $\frac{d_{\min}\lambda_{\min}(\Phi^{\top}\Phi)}{\phi_{\max}^2}$  is included. 
\end{remark}

\begin{remark}
    Note that as $\gamma \to 1$ the value $n_{\mathrm{th}}$ gets larger as $1/\ln(\gamma^{-1})$ increase. Moreover, if the ratio of minimum and maximum of stationary distribution, $d_{\max}/d_{\min}$, is large, then the value will be larger. Lastly, the bounds scales only logarithmically with key problem factors, such as the stationary distribution and the feature matrix.
\end{remark}

\begin{remark}
    In Appendix~\ref{sec:example-theorem4}, we provide an example that shows the difficulty of sharpening the bound. The example shows that the negative-definiteness of $\Phi^{\top}D^{\beta}(I-\gamma^n)(P^{\pi})^n\Phi$ does not necessarily imply the negative definiteness of $\Phi^{\top}D^{\beta}(I-\gamma^{n+1}(P^{\pi})^{n+1})\Phi$. Based on this, although the bound on $n_3^*$ in Theorem~\ref{thm:nd} may appear loose, further sharpening may be difficult. 
\end{remark}

\begin{lemma}\label{lemma:4}
Suppose that a matrix $B$ is Hurwitz stable. Then, there exists a sufficiently small $\alpha' >0 $ such that $ I + \alpha B $ is Schur stable for all $\alpha \le \alpha'$. If we define the positive real number $\alpha^{*}$ as the supremum of $\alpha^{*}$ such that $I +\alpha B$ is Schur for all $\alpha \le \alpha^{*}$, then
\[{\alpha ^*} \ge \frac{1}{{{\lambda _{\max }}(P){\lambda _{\max }}({B^T}B)}},\]
where $P \succ 0$ satisfies ${B^T}P + PB =  - I$~\cite{chen1995linear}.
\end{lemma}
\begin{proof}
If $B$ is Hurwitz stable, then by the Lyapunov argument, there exists a Lyapunov matrix $P \succ 0$ such that ${B^T}P + PB =  - I$~\cite{chen1995linear}.
Next, we have
\begin{align*}
{(I + \alpha B)^{\top}}P(I + \alpha B)
% =& P + \alpha PB + \alpha {B^{\top}}P + {\alpha ^2}{B^{\top}}PB\\
= P - \alpha I + {\alpha ^2}{B^{\top}}PB.
\end{align*}
Then, it is clear that there exists a sufficiently small $\alpha' >0 $ such that
\begin{align}
{ P - \alpha I + {\alpha ^2}{B^{\top}}PB \prec P,\label{eq:8}}
\end{align}
which implies that $I+\alpha B$ is Schur.

Moreover, a sufficient condition for~\eqref{eq:8} to hold is ${\lambda _{\max }}( - \alpha I + {\alpha ^2}{B^{\top}}PB) < 0$.
The left-hand side is bounded as
\begin{align*}
{\lambda _{\max }}( - \alpha I + {\alpha ^2}{B^{\top}}PB)
%\le& {\lambda _{\max }}( - \alpha I) + {\alpha ^2}{\lambda _{\max }}({B^{\top}}PB)\\
\le&  - \alpha  + {\alpha ^2}{\lambda _{\max }}(P){\lambda _{\max }}({B^{\top}}B).
\end{align*}
Assuming $\alpha >0$, it is equivalent to 
\begin{align*}
    - 1 + \alpha {\lambda _{\max }}(P){\lambda _{\max }}({B^{\top}}B) < 0.
\end{align*}
\end{proof}
Based on the above two results, we are now ready to establish the convergence of the algorithm~\eqref{eq:system1}.
\begin{theorem}(Convergence)\label{thm:convergence3}
Suppose that $n \geq n^*_3$ so that ${\Phi ^{\top}}{D^\beta }({\gamma ^n}{({P^\pi })^n}-I)\Phi$ is Hurwitz. Then, there exists a positive real number $\alpha^*$ such that for any  and $\alpha \leq \alpha^*$, the iterate in~\eqref{eq:system1} converges to $\theta_*^n$.
\end{theorem}
\begin{proof}
By~\cref{lemma:4}, there exists a sufficiently small $\alpha^* >0 $ such that $I-\alpha N$ is Schur stable.
\end{proof}

In this section, we have proposed a different algorithm in~\eqref{eq:system1} from the classical dynamic programming in Section~\ref{sec:dynamic-programming} and the gradient descent methods in Section~\ref{sec:gradient1}, and analyzed its convergence based on the control system perspectives~\cite{chen1995linear}. All the iterative algorithms studied until now assume that the model is already known.
In the next section, we will study model-free reinforcement learning algorithms based on these algorithms.

\section{Off-policy $n$-step TD-learning}

 For convenience, we consider a sampling oracle that takes the initial state $s_0$, and generates the sequences of states $({s_1},{s_2}, \ldots ,{s_n})$, actions $({a_0},{a_1}, \ldots ,{a_{n - 1}})$, and rewards $({r_1},{r_2}, \ldots ,{r_{n}})$ following the given constant behavior policy $\beta$. In~\cref{algo:TD1}, each states are sampled independently whereas in~\cref{algo:TD1:markov}, the states are generated following a Markov chain induced by the behavior policy $\beta$.

% In~\cref{algo:TD1}, we consider a sampling oracle that takes the initial state $s_0$, and generates the sequences of state-action pairs $({s_1},a_1),({s_2}, \ldots ,{s_n})$, actions $({a_0},{a_1}, \ldots ,{a_{n - 1}})$, and rewards $({r_1},{r_2}, \ldots ,{r_{n}})$ 
% For convenience, we consider a sampling oracle that takes the initial state $s_0$, and generates the sequences of states $({s_1},{s_2}, \ldots ,{s_n})$, actions $({a_0},{a_1}, \ldots ,{a_{n - 1}})$, and rewards $({r_1},{r_2}, \ldots ,{r_{n}})$ following the given constant behavior policy $\beta$.

The iterative algorithm in~\eqref{eq:system1} suggests us an off-policy $n$-step TD-learning algorithm ($n$-TD) given in~\cref{algo:TD1}.
Note that~\cref{algo:TD1} can be viewed as a stochastic approximation of~\eqref{eq:system1} by replacing the model parameters with the corresponding samples of the state and action. Moreover, \cref{algo:TD1} can be viewed as a standard off-policy $n$-step TD-learning with the importance sampling methods. In the model-free setting, we can employ an experience replay buffer~\cite{mnih2015human} which stores the samples collected following the behavior policy. From this experience replay buffer, we can then randomly sample, with the distribution closely approximating the stationary distribution induced by the behavior policy.

Moreover, since we are following a behavior policy different from the target policy, we use the importance sampling technique to correct this mismatch in the following manner: $\mathbb{E}\left[\frac{\pi(U\mid S)}{\beta (U\mid S)} X(U) \middle| \beta,S=s \right]= \sum_{a\in\mathcal{A}}\beta(a\mid s)\frac{\pi(a\mid s)}{\beta(a \mid s)} X(a) =\mathbb{E}\left[ X(U) \middle| \pi, S=s \right]$ for $X:\mathcal{A}\to\mathbb{R}$.

\begin{algorithm}[h]
\caption{$n$-step off-policy TD-learning}
\begin{algorithmic}[1]

\State Initialize $\theta_0\in \mathbb{R}^m$.

\For{iteration step $i \in \{0,1,\ldots\}$}

\State Sample $s_0 \sim d^{\beta}$, and sample $({s_1},{s_2}, \ldots ,{s_n})$, $({a_0},{a_1}, \ldots ,{a_{n - 1}})$, and $({r_1},{r_2}, \ldots ,{r_{n}})$ using the sampling oracle.

\State Update parameters according to
\begin{align*}
{\theta _{i + 1}} = {\theta _i} + {\alpha _i}{\rho _{n - 1}}(G - {V_{{\theta _i}}}({s_0})){\varphi(s_0)},
\end{align*}
where $\rho_{n-1} : = \prod\nolimits_{k = 0}^{n - 1} {\frac{{\pi ({a_k}|{s_k})}}{{\beta ({a_k}|{s_k})}}}$ is the importance sampling ratio, $\varphi (s) = {\Phi ^T}{e_s}$ is the $s$-th row vector of $\Phi$, $G = \sum\limits_{k = 0}^{n - 1} {{\gamma ^k}{r_{k+1}}}  + {\gamma ^n}{V_{{\theta _i}}}({s_n})$, and ${V_{{\theta _i}}}(s) = e_s^T\Phi {\theta _i}$.

\EndFor
\end{algorithmic}
\label{algo:TD1}
\end{algorithm}

\begin{algorithm}[h]
\caption{$n$-step off-policy TD-learning : Markovian observation model}
\begin{algorithmic}[1]

\State Initialize $\theta_0\in\mathbb{R}^m$.

\State Sample $(s_0,a_0,s_1,a_1,\dots,s_{n-1})$ following the Markov chain induced by behavior policy $\beta$.
\For{iteration step $i \in \{0,1,\ldots\}$}
\State Sample $a_{n-1+i}\sim\beta(\cdot\mid s_{n-1+i})$ and $s_{n+i} \sim P(\cdot \mid s_{n-1+i},a_{n-1+i})$.
\State Update parameters according to
\begin{align}
{\theta _{i + 1}} = {\theta _i} + {\alpha _i}{\rho _{i}}(G_i - {V_{{\theta _i}}}({s_i})){\varphi(s_i)},\label{eq:td-markov}
\end{align}
where $\rho_{i} : = \prod\nolimits_{k = 0}^{n - 1} {\frac{{\pi ({a_{k+i}}|{s_{k+i}})}}{{\beta ({a_{k+i}}|{s_{k+i}})}}}$ is the importance sampling ratio, $\varphi (s) = {\Phi ^T}{e_s}$ is the $s$-th row vector of $\Phi$, $G_i = \sum\limits_{k = 0}^{n - 1} {{\gamma ^k}{r_{k+i}}}  + {\gamma ^n}{V_{{\theta _i}}}({s_n})$, and ${V_{{\theta _i}}}(s) = e_s^T\Phi {\theta _i}$.

\EndFor
\end{algorithmic}
\label{algo:TD1:markov}
\end{algorithm}

Following the ideas in~\cite{borkar2000ode}, the convergence of~\cref{algo:TD1} can be easily established.
\begin{theorem}\label{thm:convergence5}
Let us consider~\cref{algo:TD1}, and assume that the step-size satisfy
\begin{align}
&\alpha_k>0,\quad \sum_{k=0}^\infty {\alpha_k}=\infty,\quad \sum_{k=0}^\infty{\alpha_k^2}<\infty.\label{eq:step-size-rule}
\end{align}
Then, $\theta_k \to \theta_*^n$ as $k \to \infty$ with probability one for any $n \ge n_3^*$, where $n_3^*$ is given in the statement of~\cref{thm:nd}.
\end{theorem}
\begin{proof}
The so-called O.D.E. model of~\cref{algo:TD1} is
\begin{align}
{{\dot \theta }_t} = {\Phi ^{\top}}{D^\beta }({T^n}(\Phi {\theta _t}) - \Phi {\theta _t}). \label{eq:ODE1}
\end{align}
By~\cref{thm:nd}, for $n \ge n_3^*$, $\Phi ^{\top} D^\beta (\gamma ^n (P^\pi)^n-I)\Phi$ is Hurwitz, and hence,~\eqref{eq:ODE1} is globally asymptotically stable. Then, the proof is completed by using of Borkar and Meyn theorem in~\cite[Thm.~2.2]{borkar2000ode}. 
\end{proof}

\cref{thm:convergence5} tells us that $n$-TD can solve the policy evaluation problem with a sufficiently large $n$.In other words, it can resolve the deadly triad problem in the case of linear function approximation.

A more practical scenario than the i.i.d. model is to consider the Markov chain $\{ (s_i,a_i) \}_{i\in\mathbb{N}}$, a trajectory of the state-action pairs following a behavior policy $\beta$. At each time $i\in\mathbb{N}$, the agent at state $s_i\in\mathcal{S}$ selects an action, $a_i\sim \beta(\cdot\mid s_i)$, and transits to the next state $s_{i+1}\sim P(\cdot\mid s_i,a_i)$. As stated in Algorithm~\ref{algo:TD1:markov}, we use the sequence of samples $(s_i,a_i,s_{i+1},a_{i+1},\dots,s_{i+n-1},a_{i+n-1},s_{i+n})$ to update $\theta_i$ via~\eqref{eq:td-markov}.

Throughout the paper, we will assume that the sequence of states observed following the behavior policy $\beta$ forms an irreducible Markov chain, i.e., there exists $k\in\mathbb{N}$ such that $\mathbb{P}[s_k=s^{\prime}|s_0=s]>0$ for any $s,s^{\prime}\in\mathcal{S}$. It also follows that the sequence of 
\begin{align}
 \tau_j:=(s_{j-1},a_{j-1}, s_j,a_j,\dots,s_{j+n-2},a_{j+n-2},s_{j+n-1} ),\label{eq:tau}  
\end{align}
 for $j\in \mathbb{N}$, which is a collection of state-action pairs observed within $n$-steps starting from the state $s_j$, also forms an irreducible Markov chain. 

Now, we can use a version of Borkar and Meyn Theorem~\cite{liu2024ode}, which applies to the Markovian observation model, to prove the convergence of Algorithm~\ref{algo:TD1:markov}.

\begin{theorem}\label{thm:markov-proof}
    Let us consider~\cref{algo:TD1:markov} and assume that the step-size $\alpha_k$ satisfies~\eqref{eq:step-size-rule} and $\lim_{k \to \infty }\left( {\frac{1}{\alpha_{k + 1}} - \frac{1}{\alpha_k}} \right)$ exists and is finite. Moreover, let us assume that Markov chain induced by behavior policy $\beta$ is irreducible. Then, $\theta_k \to \theta_*^n$ as $k \to \infty$ with probability one for any $n \ge \bar n^*$, where $\bar n^*$ is given in the statement of~\cref{thm:nd}.
\end{theorem}
The proof is given in~\cref{app:proof:thm:markov-proof}. \cref{thm:markov-proof} indicates that $n$-step TD is guaranteed to converge under more practical Markovian observation model.

\begin{remark}
    The framework can be easily extended to the eposidic MDP as follows: an episodic MDP can be interpreted as a non-episodic MDP by introducing the terminal state. In this case, the MDP is fundamentally a non-episodic MDP but when visiting the terminal state, all the remaining rewards become zero. Therefore, in this setting we can implement the $n$-step TD as follows: 1) when the episode ends before $n$ time steps, then we can use the return without using the bootstrapping 2) when the episode did not end before $n$ time steps, then we can use the bootstrapping after $n$ time steps.
\end{remark}

\begin{remark}
    We note that the bound in~\cref{thm:nd} is only a sufficient condition, and as long as $\Phi^{\top}(\gamma^n (P^{\pi})^n-I)D^{\beta})\Phi$ is Hurwitz, the algorithm will converge.
\end{remark}

\section{Conclusion}

In this paper, we have investigated the convergence and properties of $n$-step TD-learning algorithms. We have proved that under the deadly triad scenario, the $n$-step TD-learning algorithms converge to useful solutions as the sampling horizon $n$ increases sufficiently. We have comprehensively examined the fundamental properties of their model-based deterministic counterparts, which can be viewed as prototype deterministic algorithms whose analysis plays a pivotal role in understanding and developing their model-free RL counterparts. Based on the analysis and insights from the deterministic algorithms, we have established convergence of two $n$-step TD-learning algorithms. 
\bibliographystyle{plain}
\bibliography{reference}

\appendix

\section{Derivation of main results}

\begin{lemma}\label{lemma:Schur-eq}
Let us consider an affine mapping $h:{\mathbb R}^m \to {\mathbb R}^m$ defined as $h(x)= Bx +b$, where $B\in {\mathbb R}^{m\times m}$ and $b\in {\mathbb R}^m$ are constants. Then, $h$ is a contraction if and only if $B$ is Schur.
\end{lemma}
\begin{proof}
First of all, for any norm $||\cdot||$, since $||h(x)-h(y)|| = ||Bx+b-By-b||=||Bx-By||$, one can prove that $h$ is a contraction if and only if $g(x):= Bx$. Therefore, let us focus on $g$, and for sufficiency, let us suppose that $B$ is Schur, which means $\rho(B)< 1$, where $\rho$ is the spectral radius. By Lyapunov theorem~\cite{chen1995linear}, $B$ is Schur if and only if there exists a symmetric positive definite matrix $P \succ 0$ such that $B^\top P B \preceq \alpha P$ and $\sqrt {\lambda_{\max}(P)} = 1$, where $\alpha <1$. One can easily prove that this implies $\sqrt {\lambda _{\max}(B^ \top PB)}  \le \alpha \sqrt {\lambda_{\max}(P)}  = \alpha$. By defining the matrix norm $\| B \| = \sqrt {\lambda_{\max }(B^ \top P B)}$, we have $\| B \| \le \alpha < 1$. For such a norm, it follows that $||g(x)-g(y)|| = ||Bx-By|| \le ||B||\cdot ||x-y|| \le \alpha ||x-y||$, where $||\cdot ||$ denotes both the matrix and vectors norms that are compatible with each other. This implies that $g$ is a contraction. {In order to prove the necessity part}, let us suppose that $g$ is a contraction but $B$ is not Schur. Then, for any matrix norm $\|\cdot \|$, we should have $||B||\geq\rho(B)\geq 1$, where $\rho$ means the spectral radius. On the other hand, by the hypothesis, one can choose a vector norm $\|\cdot \|$ such that $||g(x)-g(y)|| = ||Bx-By|| \le \alpha ||x-y||$ with $\alpha <1$, which leads to $\frac{||B(x - y)||}{||x - y||} \le \alpha$. This implies that the corresponding induced matrix norm $\|\cdot \|$ should satisfy $||B|| \le \alpha  < 1$, which is a contradiction. Therefore, $B$ is Schur, completing the proof.
\end{proof}
\begin{lemma}\label{lemma:Schur-eq2}
Let us consider an affine mapping $h:{\mathbb R}^m\to {\mathbb R}^n$ defined as $h(x)= Bx +b$, where $B\in {\mathbb R}^{m\times m}$ and $b\in {\mathbb R}^m$ are constants. Then, $h$ has a unique fixed point and the iterate $x_{k+1} = h(x_k), k\in \{0,1,\ldots \}$ converges to the fixed point if and only if $h$ is a contraction.
\end{lemma}
\begin{proof}
For the sufficiency, suppose that $h$ is a contraction. Then, the proof is completed by the Banach fixed point theorem.
For the necessity, suppose that $h$ has a unique fixed point $x^*$ and the iterate $x_{k+1} = h(x_k), k\in \{0,1,\ldots \}$ converges to the fixed point. Then, $h(x^*)= Bx^* + b = x^*$, and defining the shifted iterate $y_k = x_k - x^*$, one gets $y_{k+1} = B y_k, k\in \{0,1,\ldots \}$, where $y_k \to 0$ as $k \to \infty$. By the standard linear system theory~\cite{chen1995linear}, this implies that $B$ is Schur.
Then,~\cref{lemma:Schur-eq} leads to the conclusion that $h$ is a contraction. 
\end{proof}

\begin{remark}
Note that Banach fixed point theorem is a sufficient condition in general. \cref{lemma:Schur-eq2} tells us that it is also a necessary condition for affine mappings. The sufficiency part also follows from Proposition 2.1.1 in~\cite{bertsekas2022abstract}.
\end{remark}

\section{O.D.E. approach for stochastic approximation}
Let us suppose that $\{Y_k\in\mathcal{Y}\}_{k\in\mathbb{N}}$ is a stochastic process induced by a irreducible Markov chain with a unique stationary distribution $\mu$ and finite space $\mathcal{Y}$. In this subsection, we study convergence of the stochastic recursion 
\begin{align}
    x_{k+1} = x_k +\alpha_k h(x_k,Y_k),\label{eq:sa}
\end{align}
where $h:\mathbb{R}^m\times\mathcal{Y}\to\mathbb{R}^m$ is a mapping, and $k\in\mathbb{N}$ is the iteration step. Convergence of the above update is closely related to the corresponding ODE
\begin{align}
\frac{d}{dt}x_t = \bar{h}(x_t),\quad t\in\mathbb{R}_+, \quad x_0 \in \mathbb{R}^m,\label{eq:ode2}
\end{align}
where $\bar{h}(x) = \mathbb{E}_{Y\sim \mu}[h(x,Y)]$. An important concept studying the convergence of~(\ref{eq:sa}) is the so-called re-scaled map : $h_c(x,y) := \frac{h(cx,y)}{c}$, where $x\in\mathbb{R}^m$ and $y\in\mathcal{Y}$. Moreover, let us denote the limit of the re-scaled map, if exists, as 
\begin{align}
h_{\infty}(x,y):=\lim_{c\to\infty}h_c(x,y). \label{eq:h1}
\end{align}
Likewise, let us define the re-scaled mapping and the corresponding limit, if exists, as
\begin{align}
\bar{h}_c(x):=\frac{\bar{h}(cx)}{c},\quad \bar{h}_{\infty}(x):=\lim_{c\to\infty}\bar{h}_c(x). \label{eq:h2}
\end{align}
We introduce essential assumptions required to guarantee the convergence of the stochastic algorithm~\eqref{eq:sa}.
\begin{assumption}\label{assmp:borkar-meyn-markov}
$\,$
    \begin{enumerate}
    \item[1)] The limits defined in~\eqref{eq:h1} and~\eqref{eq:h2} exist.
        \item[2)] There exist functions $b:\mathbb{R}^m\times\mathcal{Y}\to\mathbb{R}$ and $\kappa:\mathbb{R}_+\to\mathbb{R}$ such that for $c\in\mathbb{R}_+$:
        \begin{align*}
         h_c(x,y)-h_{\infty}(x,y) =& \kappa(c) b(x,y),
        \end{align*}
        and $\lim_{c\to\infty}\kappa(c)=0$. Furthermore, there exists a function $L_b:\mathcal{Y}\to\mathbb{R}$ such that $\mathbb{E}_{y\sim\mu}[L_b(y)]$ is finite and $||b(x,y)-b(x^{\prime},y)||\leq L_b(y)||x-x^{\prime}||$.
        \item[3)] The mapping $h:\mathbb{R}^m\to\mathbb{R}^m$ is Lipschitz continuous:
        \begin{align*}
            ||h(x,y)-h(x^{\prime},y)|| \leq & L(y) ||x-x^{\prime}||,\\
            ||h_{\infty}(x,y)-h_{\infty}(x^{\prime},y)|| \leq & L(y) ||x-x^{\prime}||,
        \end{align*}
        where $L:\mathcal{Y}\to\mathbb{R}$, and the quantities, $\bar{h}(x)$, $\bar{h}_{\infty}(x)$ and $ L = \mathbb{E}_{y\sim\mu}[L(y)]$ are finite for any $x\in\mathbb{R}^d$.
        \item[4)] The re-scaled mapping $\bar{h}_c(x)$ converges uniformly to $\bar{h}_{\infty}(x)$ on any compact subsets of $\mathbb{R}^m$. Moreover, the origin of the following ODE is globally asymptotically stable:
        \begin{align*}
            \frac{d}{dt}x_t = \bar{h}_{\infty}(x_t),\quad x_0\in\mathbb{R}^m.
        \end{align*}
        \item[5)] The ODE in~\eqref{eq:ode2} has a unique globally asymptotically stable equilibrium point $x^*\in\mathbb{R}^m$ such that $\bar{h}(x^*)=0$.
        \item[6)] The step-size condition satisfies~(\ref{eq:step-size-rule}) and
        \begin{align*}
         \lim_{k\to\infty}\left(\frac{1}{\alpha_{k+1}}-\frac{1}{\alpha_k}\right)
        \end{align*}
 exists and is finite.
        \item[7)] The stochastic process $\{Y_i\}_{i\in\mathbb{N}}$ is generated by a irreducible Markov chain.
    \end{enumerate}
\end{assumption}

\begin{remark}
     $\{Y_i\}_{i\in\mathbb{N}}$ does not need to be necessarily aperiodic. That is, we only require the Markov chain induced by the behavior policy $\beta$ to be irreducible. 
\end{remark}

\begin{lemma}[Corollary 1 in~\cite{liu2024ode}]\label{lem:borkar-markov}
   Suppose that~\cref{assmp:borkar-meyn-markov} holds. Then $x_k$ converges to $x^*$ with probability one, where $x^*$ is defined in the item (5) of~\cref{assmp:borkar-meyn-markov}.
\end{lemma}

\section{Proof of Theorem~\ref{thm:markov-proof}}\label{app:proof:thm:markov-proof}
\begin{proof}
    Note that the update in~(\ref{eq:td-markov}) can be written as,
\begin{align*}
    \theta_{i+1}=\theta_i+\alpha_i h(\theta_i,\tau_{i+1}),
\end{align*}
where $h(\theta,\tau)  := \rho(G- \varphi(s_k)^{\top}\theta)\varphi(s_k)$ and $\tau=(s_k,a_k,s_{k+1},a_{k+1},\cdots,s_{k+n},a_{k+n})$ is the $n$-step state-action pair sequence, and $\rho$ and $G$ are the corresponding importance sampling ratio and the discounted sum of return defined in~\cref{algo:TD1:markov}, respectively.
Meanwhile, the re-scaled map becomes 
\begin{align*}
    h_c(\theta,\tau)
    =& \rho\left(\frac{\sum_{i=0}^{n-1} \gamma^k r_{i+k}}{c}+\gamma^n\varphi(s_{k+n})^{\top}\theta-\varphi(s_k)^{\top}\theta\right)\varphi(s_k),
\end{align*}
and as $\frac{\sum_{k=0}^{n-1} \gamma^k r_{i+k}}{c}\to 0$ as $c\to\infty$, $h_c(\theta,\tau)$ uniformly converges to $ h_{\infty}(\theta,\tau)$ where
\begin{align*}
      h_{\infty}(\theta,\tau)=&\rho (\gamma^n\varphi(s_{k+n})^{\top}\theta-\varphi(s_k)^{\top}\theta)\varphi(s_k).
\end{align*}
Likewise, we have $
    \bar{h}_c(\theta):= \Phi^{\top}D^{\beta}\left(\frac{T^n(\Phi(c\theta))}{c} -\Phi\theta \right).$ As $\tau:=\{(s_k,a_k)\}_{k=0}^{n}$ is sampled from its stationary distribution, $(s_0,a_0)$ is sampled from its stationary distribution, which corresponds to the diagonal elements of $D^{\beta}$. Then, $\mathbb{E}\left[ \varphi(s_0)\varphi(s_0) \right]= \Phi^{\top} D^{\beta} \Phi$. Likewise, $\mathbb{P}(s_{n}\mid s_0)$ is the probability distribution that corresponds to $e^{\top}_{s_n}(P^{\pi})^n$. Therefore, we have $\mathbb{E}\left[  \rho \varphi(s_0) \varphi(s_{n})^{\top} \middle| \beta \right]= \Phi^{\top}D^{\beta}(P^{\pi})^n\Phi $.

Moreover, noting that 
\begin{align*}
    \lim_{c\to\infty}\frac{T^n(\Phi(c\theta))}{c}=& \lim_{c\to\infty} \frac{\sum^{n-1}_{i=0}\gamma^i (P^{\pi})^iR^{\pi}+\gamma^n(P^{\pi})^n\Phi(c\theta)}{c} \\
    =&  \gamma^n (P^{\pi})^n \Phi\theta,
\end{align*} 
$\bar{h}_c(\theta)$ converges to $h_{\infty}(\theta)$ uniformly in $\theta$ where 
\begin{align*}
     \bar{h}_{\infty}(\theta):=& \Phi^{\top}D^{\beta}(\gamma^n(P^{\pi})^n\Phi\theta-\Phi\theta).
\end{align*}
Now, let us check the conditions in~\cref{assmp:borkar-meyn-markov} to apply~\cref{lem:borkar-markov}. Item (1) has been addressed above. Since we consider a linear mapping, the conditions in items (2) and (3) of~\cref{assmp:borkar-meyn-markov} are naturally satisfied. Items (4) and (5) come from~\cref{thm:nd}. Items (6) and (7) follow from the assumption on the step-size condition and the fact that $\{\tau_i\}_{i\in\mathbb{N}}$ forms an irreducible Markov chain.
\end{proof}

\section{Experiments}

This section presents experimental results that validate the theoretical findings. We use a MDP with two states ($|\mathcal{S}|=2$), two actions ($|\mathcal{A}|=2$), and a discount factor of $\gamma = 0.99$. The reward is set as\textbf{} zero. The behavior policy, $\beta$, selects actions with probabilities $\beta(1\mid s)=0.76$ and $\beta(2\mid s)=0.24$, while the target policy, $\pi$, uses probabilities $\pi(1\mid s)=0.3$ and $\pi(2\mid s)=0.7$ for all states $s\in\mathcal{S}$. The feature matrix and the transition matrix are as follows:
\begin{align*}
    \Phi = \begin{bmatrix}
        1.78\\
        1.2
    \end{bmatrix},\quad P =\begin{bmatrix}
        0.58 & 0.42\\
        0.92 & 0.08 \\
        0.47 & 0.53\\
        0.6 & 0.4
    \end{bmatrix}
\end{align*}
The matrix $\Phi^{\top}D^{\beta}(\gamma^n(P^{\pi})^n-I)\Phi$ becomes a Hurwitz matrix when $n\geq 3$, but not when $n$ is 1 or 2. This is illustrated in Figure~\ref{fig:exp}, which shows that the algorithm diverges for $n\in\{1,2\}$ and converges for $n\in\{3,4\}$.

% \begin{table}[!ht]
%     \begin{center}
% \begin{tabular}{|c|c|c|}
% \hline
%              &  Minimum value  &  $n_i^*$  \\ \hline
%      Lemma 5 (1)  & 3  & 11   \\ \hline
%      Lemma 5 (2) & 5  &  11 \\ \hline
%      Theorem     & 3  & 54  \\\hline      
% \end{tabular}
%     \end{center}
%     \label{tab: }\caption{Although the bounds are not tight, they scale logarithmically with respect to each factor. The second column calculates the upper bound on $n_i^*$ computed in each Lemma or Theorems, for $i\in 1,2,3$. }
% \end{table}
Meanwhile, the values of the upper bound for $n_1^*,n_2^*$ and $n_3^*$ in Lemma~\ref{lem:A_Schur} and Theorem~\ref{thm:nd} are $11,11,54$, respectively. The minimum values for the matrix $A$ to be Schur is $3$, $\Pi T^n$ to be a contraction is 5, and $\Phi^{\top}D^{\beta}(\gamma^n(P^{\pi})^n-I)\Phi$ to be a Hurwitz matrix is 3. Even though the bounds are not tight, the bound only scales logarithmically to important key factors.

\begin{figure}[ht]
  \centering
  \includegraphics[width=0.41\textwidth]{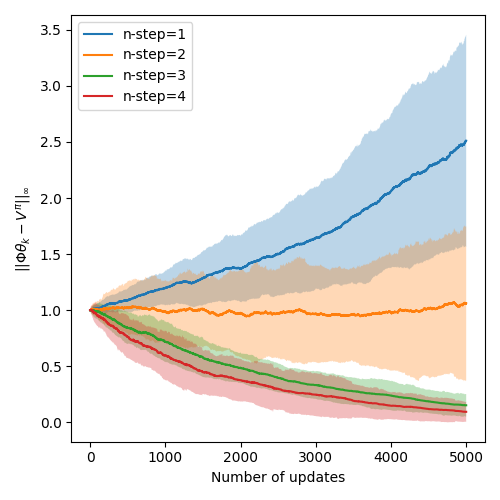}
  \caption{Due to high variance, we clipped the importance sampling ratio. Instead of $\rho$, we used $\min\{\rho,9\}$.}
  \label{fig:exp}
\end{figure}

\section{Examples for Theorem 4}\label{sec:example-theorem4}

In this section, we provide an example where $\Phi^{\top}D^{\beta}(\gamma P^{\pi}-I)\Phi$ is Hurwitz but $\Phi^{\top}D^{\beta}(\gamma^2(P^{\pi})^2-I)\Phi$ is not. Consider a MDP with $\beta(1\mid s)=0.8,\beta(2\mid s)=0.2$ and $\pi(1\mid s)=0.3,\pi (2\mid s)=0.7$ for $s\in\mathcal{S}$, $\gamma=0.99$ and 
\begin{align*}
    \Phi = \begin{bmatrix}
        2 \\
        1\\
        2
    \end{bmatrix},\quad P = \begin{bmatrix}
        0.1 & 0.8 & 0.1\\
        0 & 0.5 & 0.5 \\
        0 & 0 & 1 \\
        0 & 0 & 1\\
        0.8 & 0 & 0.2\\
        0 & 0 & 1
    \end{bmatrix}
\end{align*}
Then, one can see that $\Phi^{\top}D^{\beta}(\gamma P^{\pi}-I)\Phi\approx -0.17$ whereas $\Phi^{\top}D^{\beta}(\gamma^2 (P^{\pi})^2 -I)\Phi\approx 0.02$.
\section{Example for Remark~\ref{remark1}}\label{example:compare}
Consider the following MDP where $\beta(1\mid s)=0.6 ,\;\beta(2\mid s)=0.4$ and $\pi(1\mid s)=0.7, \; \pi(2 \mid s)=0.3$ for $s\in\mathcal{S}$. The discount factor $\gamma$ is 0.99. The feature matrix and transition matrix are 
\begin{align*}
    \Phi = \begin{bmatrix}
            1\\
            -2
    \end{bmatrix}, \quad P = \begin{bmatrix}
        0.3 & 0.7\\
        0 & 1\\
        0.9 & 0.1\\
        0.8 & 0.2
    \end{bmatrix}.
\end{align*}Then, we can see that $\frac{\ln\left( \frac{d_{\min}\lambda_{\min}(\Phi^{\top}\Phi)}{\phi_{\max}^2} \right)}{\ln \gamma}\approx 48 $ while $\frac{\ln\left(  \frac{d_{\min}\lambda_{\min}(\Phi^{\top}\Phi)}{{{d_{\max }}{\lambda _{\max }}({\Phi ^{\top}}\Phi )}}\frac{1}{{\sqrt {|\mathcal{S}|} }}\right)}{\ln \gamma}\approx 37 $. The example in Section~\ref{sec:example-theorem4} provides the other case.

\subsection{Convergence of $n$-step TD and $n$-step PVI}
\begin{example}[MDP2]
    
We consider a discounted MDP with state space $|\mathcal{S}|=4$,
action space and $|\mathcal{A}|=3$, and discount factor $\gamma=0.99$. In this MDP instance, for $n=2$ we observe that $n$-step TD does not converge,
whereas $n$-step approximate value iteration (AVI) converges.  The feature dimension is two, given by
\[
\Phi
=
\begin{bmatrix}
-0.25 & -1.20\\
\phantom{-}0.08 & -0.83\\
-0.05 & -1.04\\
\phantom{-}0.01 & -1.17
\end{bmatrix}.
\]

\textbf{Target and behavior policies.}
The target policy $\pi$ and behavior policy $\beta$ are state-dependent distributions over $\mathcal A$:
\[
\pi(\cdot\mid s)=
\begin{bmatrix}
0.51 & 0.47 & 0.02\\
0.05 & 0.47 & 0.48\\
0.79 & 0.16 & 0.05\\
0.02 & 0.13 & 0.85
\end{bmatrix},
\quad
\beta(\cdot\mid s)=
\begin{bmatrix}
0.20 & 0.04 & 0.76\\
0.41 & 0.52 & 0.07\\
0.87 & 0.08 & 0.05\\
0.03 & 0.95 & 0.02
\end{bmatrix}.
\]

\textbf{Transition kernel.}
Let $P\in\mathbb R^{|\mathcal S||\mathcal A|\times |\mathcal S|}$ denote the transition matrix
whose row indexed by $(s,a)$ specifies $P(\cdot\mid s,a)$.
With the row ordering $(s,a)=(1,1),(1,2),(1,3),(2,1),\dots,(4,3)$, we set
\[
P=
\begin{bmatrix}
0.19 & 0.04 & 0.04 & 0.73\\
0.90 & 0.03 & 0.05 & 0.02\\
0.06 & 0.86 & 0.05 & 0.03\\
0.08 & 0.51 & 0.32 & 0.09\\
0.03 & 0.04 & 0.24 & 0.69\\
0.03 & 0.04 & 0.11 & 0.82\\
0.03 & 0.03 & 0.91 & 0.03\\
0.14 & 0.06 & 0.15 & 0.65\\
0.72 & 0.23 & 0.02 & 0.03\\
0.14 & 0.29 & 0.48 & 0.09\\
0.07 & 0.22 & 0.14 & 0.57\\
0.32 & 0.03 & 0.36 & 0.29
\end{bmatrix}.
\]

\textbf{Reward matrix.}
We use a state-to-next-state reward matrix $R\in\mathbb R^{|\mathcal S|\times|\mathcal S|}$,
where the realized reward upon transitioning $s\to s'$ is $r(s,s')=R_{s,s'}$:
\[
R=
\begin{bmatrix}
\phantom{-}0.95 & \phantom{-}0.64 & \phantom{-}0.26 & -0.60\\
-0.69 & \phantom{-}0.33 & \phantom{-}0.60 & \phantom{-}0.16\\
\phantom{-}0.37 & -0.02 & \phantom{-}0.80 & \phantom{-}0.37\\
-0.22 & \phantom{-}0.82 & \phantom{-}0.49 & -0.24
\end{bmatrix}.
\]

\end{example}

\begin{example}[MDP3]

We consider a discounted MDP with state space $|\mathcal{S}|=4$,
action space $|\mathcal{A}|=2$, and discount factor $\gamma=0.99$.
In this MDP instance, for $n=3$ we observe that $n$-step AVI does not converge,
whereas $n$-step TD converges.
The feature dimension is two, given by
\[
\Phi
=
\begin{bmatrix}
\phantom{-}0.27 & -0.30\\
\phantom{-}0.00 & \phantom{-}0.11\\
-0.85 & -0.01\\
-0.15 & \phantom{-}0.49
\end{bmatrix}.
\]

\textbf{Target and behavior policies.}
The target policy $\pi$ and behavior policy $\beta$ are state-dependent distributions over $\mathcal A$:
\[
\pi(\cdot\mid s)=
\begin{bmatrix}
0.99 & 0.01\\
0.45 & 0.55\\
0.25 & 0.75\\
0.02 & 0.98
\end{bmatrix},
\quad
\beta(\cdot\mid s)=
\begin{bmatrix}
0.01 & 0.99\\
0.01 & 0.99\\
0.04 & 0.96\\
0.99 & 0.01
\end{bmatrix}.
\]

\textbf{Transition kernel.}
Let $P\in\mathbb R^{|\mathcal S||\mathcal A|\times |\mathcal S|}$ denote the transition matrix
whose row indexed by $(s,a)$ specifies $P(\cdot\mid s,a)$.
With the row ordering $(s,a)=(1,1),(1,2),(2,1),(2,2),(3,1),(3,2),(4,1),(4,2)$, we set
\[
\scriptsize
P=
\begin{bmatrix}
0.11 & 0.01 & 0.22 & 0.66\\
0.97 & 0.02 & 0.00 & 0.01\\
0.07 & 0.16 & 0.02 & 0.75\\
0.02 & 0.02 & 0.02 & 0.94\\
0.01 & 0.72 & 0.03 & 0.24\\
0.06 & 0.01 & 0.69 & 0.24\\
0.97 & 0.01 & 0.00 & 0.02\\
0.27 & 0.01 & 0.57 & 0.15
\end{bmatrix}.
\]

\textbf{Reward matrix.}
We use a state-to-next-state reward matrix $R\in\mathbb R^{|\mathcal S|\times|\mathcal S|}$,
where the realized reward upon transitioning $s\to s'$ is $r(s,s')=R_{s,s'}$:
\[
R=
\begin{bmatrix}
-0.33 & \phantom{-}0.66 & -0.34 & \phantom{-}0.88\\
\phantom{-}0.11 & -0.44 & -0.86 & \phantom{-}0.58\\
\phantom{-}0.26 & -0.72 & -0.57 & \phantom{-}0.13\\
-0.36 & -0.37 & \phantom{-}0.20 & -0.42
\end{bmatrix}.
\]

\end{example}

    \begin{figure}
        \centering
        \begin{subfigure}[t]{0.34\textwidth}
            \centering
            \includegraphics[width=\linewidth]{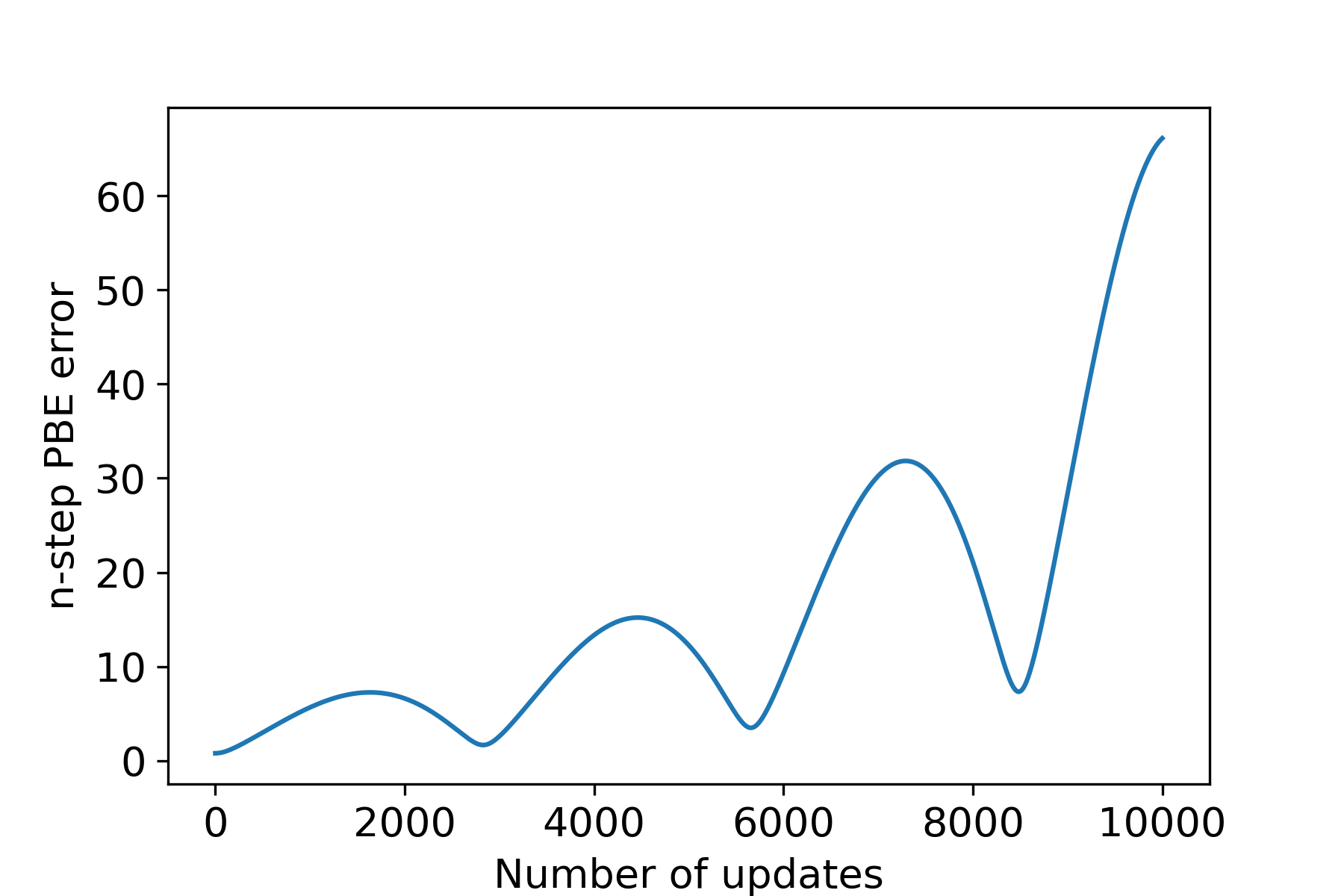}
            \caption{ $n$-step TD}
        \end{subfigure}
        \hspace{0.04\textwidth}
        \begin{subfigure}[t]{0.34\textwidth}
            \centering
            \includegraphics[width=\linewidth]{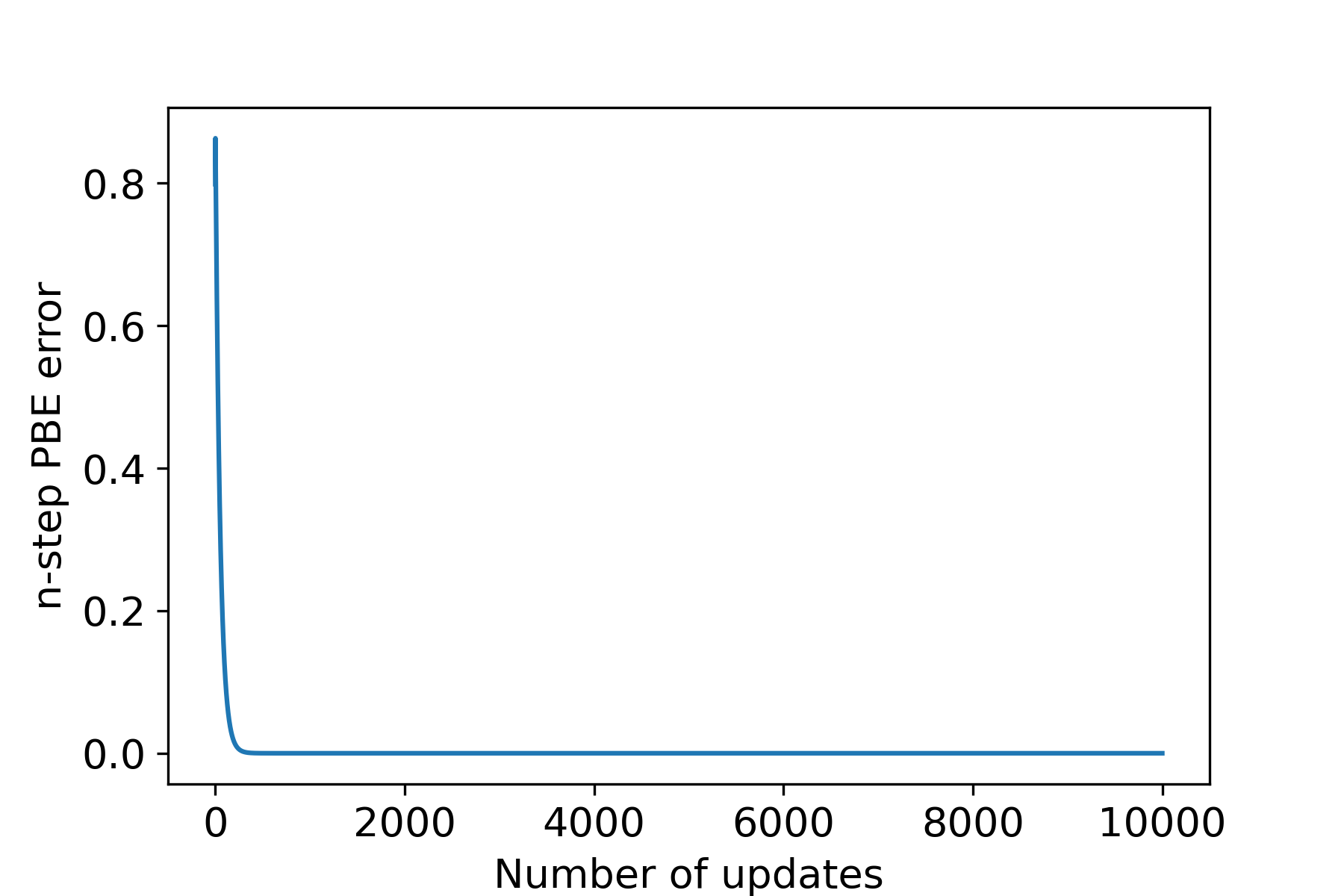}
            \caption{ $n$-step PVI}
        \end{subfigure}
        \begin{subfigure}[t]{0.34\textwidth}
            \centering
            \includegraphics[width=\linewidth]{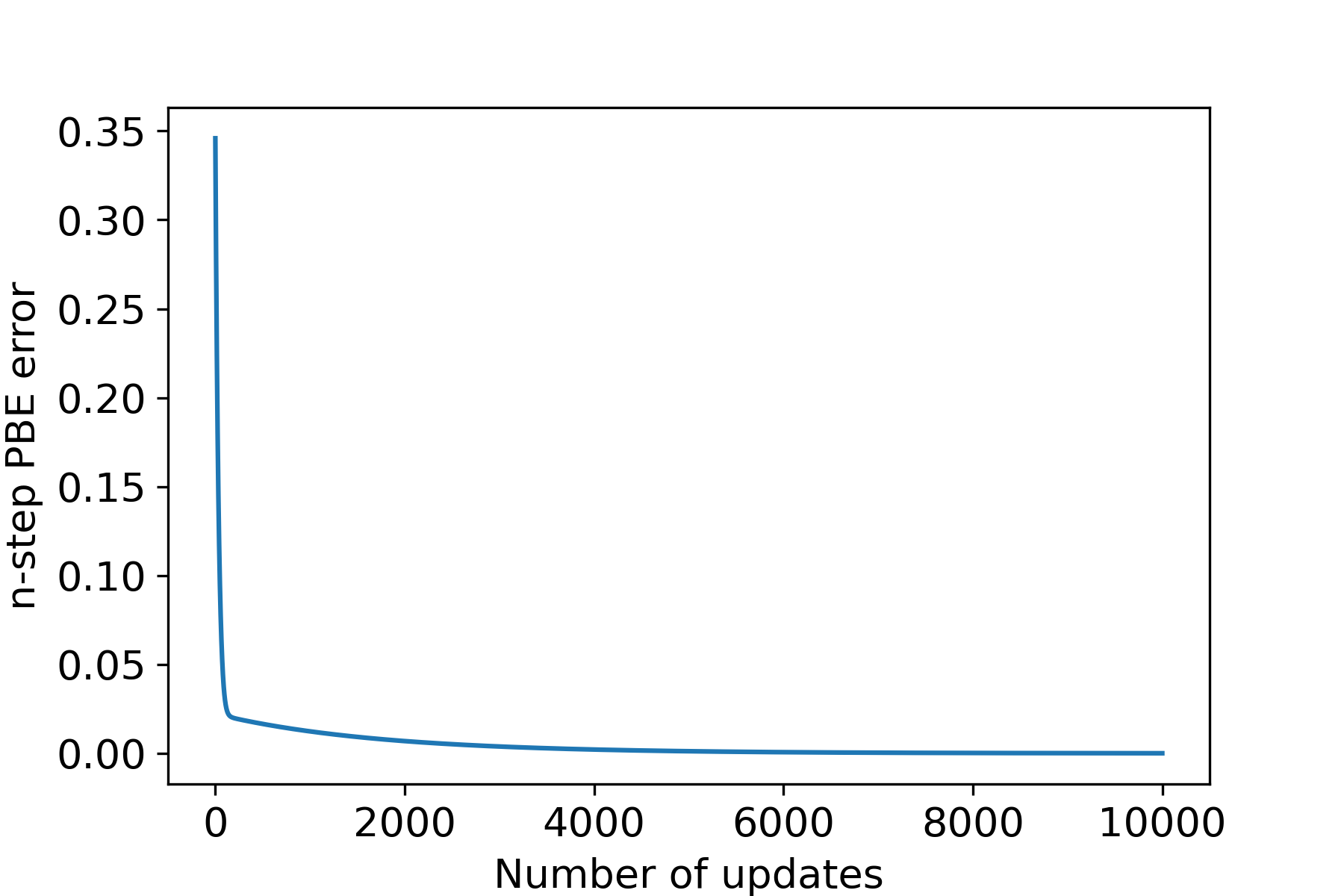}
            \caption{$n$-step TD}
        \end{subfigure}
        \hspace{0.04\textwidth}
        \begin{subfigure}[t]{0.34\textwidth}
            \centering
            \includegraphics[width=\linewidth]{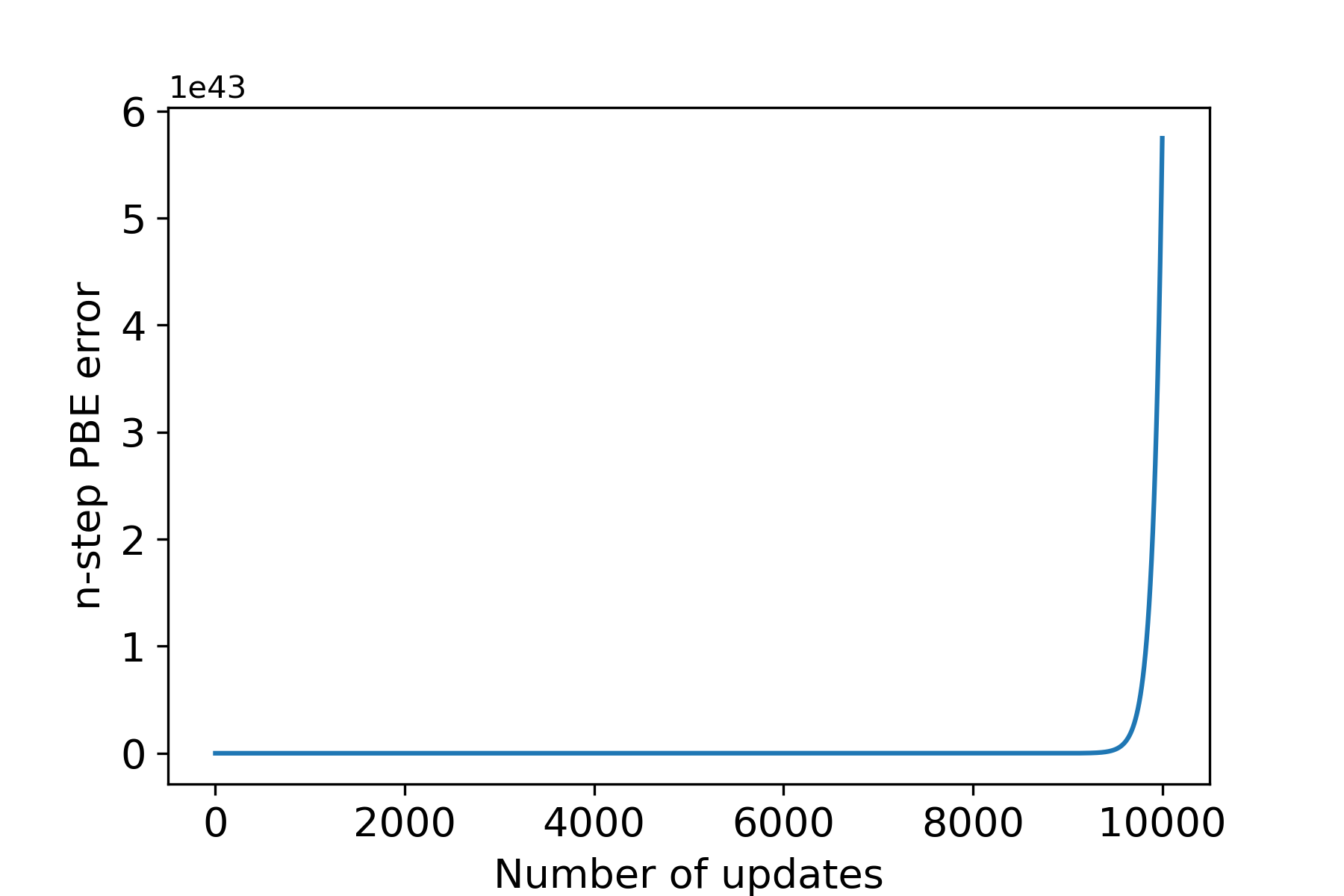}
            \caption{$n$-step PVI}
        \end{subfigure}

    \end{figure}

\end{document}